%% file: main.tex
\documentclass[twoside,11pt]{article}

\usepackage{graphicx}
\usepackage{subcaption}
\usepackage{wrapfig}

\usepackage[preprint]{jmlr2e}

\usepackage[utf8]{inputenc} 
\usepackage[T1]{fontenc}    
\usepackage{hyperref}       
\usepackage{url}            
\usepackage{booktabs}       
\usepackage{amsfonts}       
\usepackage{nicefrac}       
\usepackage{microtype}      
\usepackage{xcolor}         

\usepackage{enumerate}
\usepackage{natbib}
\usepackage{amsmath}
\usepackage{amsthm}
\usepackage{cleveref}
\usepackage{amssymb}
\usepackage{bbm}
\usepackage{algorithm}
\usepackage{algorithmic}
\usepackage{mathtools}
\usepackage{autonum}

\usepackage{notations}

\jmlrheading{}{}{}{}{}{}{Boris Muzellec and Francis Bach and Alessandro Rudi}

\ShortHeadings{Learning PSD-Valued Functions Using Kernel SoS}{Muzellec and Bach and Rudi}

\firstpageno{1}

\begin{document}


\author{\name Boris Muzellec \email boris.muzellec@inria.fr \\
       \name Francis Bach \email francis.bach@inria.fr \\
       \name Alessandro Rudi \email alessandro.rudi@inria.fr \\
       \addr INRIA Paris, 2 rue Simone Iff, 75012, Paris, France \\
			 ENS - Département d’Informatique de l’École Normale Supérieure, \\
		      PSL Research University, 2 rue Simone Iff, 75012, Paris, France}

\editor{\BM{editor}}

\title{Learning PSD-Valued Functions Using Kernel Sums-of-Squares}

\maketitle

\begin{abstract}
Shape constraints such as positive semi-definiteness (PSD) for matrices or convexity for functions play a central role in many applications in machine learning and sciences, including metric learning, optimal transport, and economics. Yet, very few function models exist that enforce PSD-ness or convexity with good empirical performance and theoretical guarantees. In this paper, we introduce a kernel sum-of-squares model for functions that take values in the PSD cone, which extends kernel sums-of-squares models that were recently proposed to encode non-negative scalar functions. We provide a representer theorem for this class of PSD functions, show that it constitutes a universal approximator of PSD functions, and derive eigenvalue bounds in the case of subsampled equality constraints. We then apply our results to modeling convex functions, by enforcing a kernel sum-of-squares representation of their Hessian, and show that any smooth and strongly convex function may be thus represented. Finally, we illustrate our methods on a PSD matrix-valued regression task, and on scalar-valued convex regression. 

\end{abstract}
\begin{keywords}
 Positive-definite matrices, kernels, sums of squares, convex functions
\end{keywords}

\input{sections/intro}

\input{sections/background}
\input{sections/positive_operators}
\input{sections/convex_functions}

\input{sections/applications}

\input{sections/conclusion}

\acks{BM wishes to thank Pierre-Cyril Aubin-Frankowski for conversations related to this work and comments on the manuscript. This work was funded in part by the French government under management of Agence Nationale de la Recherche as part of the “Investissements d’avenir” program, reference ANR-19-P3IA-0001(PRAIRIE 3IA Institute). We also acknowledge support from the European Research Council (grants SEQUOIA 724063 and REAL 947908), and support by grants from Région Ile-de-France.}



\appendix

\input{sections/appendix}



\vskip 0.2in
\bibliography{references}



\end{document}

%% file: sections/intro.tex
\section{Introduction}

Linear models, and kernel methods in particular, were proved over the past few decades to be of great effectiveness in machine learning, both in supervised and unsupervised learning~\citep[see, e.g.,][]{scholkopf2002learning}. Indeed, they offer an attractive trade-off between representation power, algorithmic simplicity, and theoretical guarantees. Moreover, they constitute a flexible toolbox that is not restricted to vector inputs or scalar outputs:  years of kernel design have resulted in kernels that may be used to learn on a vast diversity of data types, or that may go beyond scalar functions and have vector outputs~\citep{carmeli2010vector}.

Yet, linear models have had limited applications to tasks in which functions must satisfy a given shape constraint, such as monotonicity or convexity, on their whole domain, and not only on the training points. A notable exception is non-negativity: recently, \cite{marteau2020non} proposed a kernel sum-of-squares (SoS) model for smooth non-negative scalar functions with universal approximation properties, which was then applied to non-convex optimization by \cite{rudi2020global}, to the estimation of optimal transport distances by \cite{vacher2021dimension}, and to optimal control~\citep{berthier2021infinite}. Further, \cite{marteau2020non} propose extensions of their model to polyhedral cone constraints. However, some important classes of shape constraints may not be encoded as polyhedral cones: in particular, neither the cone of positive semi-definite (PSD) matrices nor the cone of convex functions are polyhedral. Yet, they are at the heart of numerous applications in machine learning and sciences: PSD and convexity constraints play a central role in optimal transport~\citep{brenier1991polar}, monopolistic games in economics~\citep{rochet1998ironing,mirebeau2016adaptive} and Newton's least resistance problem in physics~\citep{lachand2005minimizing}.

Recently, approaches to handle convexity constraints have been proposed within the deep learning community~\citep{amos2017input}, and were applied to optimal control~\citep{chen2018optimal} and optimal transport~\citep{makkuva2020optimal}, among others. However, due to lack of results regarding their approximation properties, those model allow very limited -- if any -- theoretical guarantees when applied to those problems. In the kernel world, \cite{aubin2021handling} proposed a framework to enforce shape constraints in RKHSs, which comes with performance guarantees. Their method relies on a tightening of cone constraints through compact coverings in vector-valued RKHSs (vRKHSs), i.e.\ they apply a finite number of SDP inequalities involving the Hessian $\bH_f$ of the learnt function $f$, of the type
$$
\eta_m \|f\|_\hh \eye_P \preceq \diag(\bb) + \bH_f(\tilde{x}_m),~~~ m \in [M],
$$
with suitable $\eta_m > 0$ and $\tilde{x}_m \in \Xcal$. This allows applying various shape constraints for scalar or vector-valued functions, such as convexity or monotonicity. Alternatively, \cite{curmei2020shape} proposed a method to fit polynomial functions under a class of shape constraints that subsumes convexity and monotonicity over box domains. Their approach relies on representing derivatives and Hessians as polynomial sum-of-squares matrices, and allows the authors to obtain statistical consistency results.

In comparison, we propose in this paper a linear model for matrix-valued functions that is guaranteed to take PSD values. This model may be seen as an extension of the sums-of-squares model of \cite{marteau2020non} to a ``sums-of-outer-products'' model that allows representing smooth functions taking values in the (non-polyhedral) PSD cone, or as an adaptation of \cite{curmei2020shape} to kernel (instead of polynomial) sums-of-squares. Further, our model inherits from the interpolation inequalities on constraints violations which where shown in the scalar case by \cite{rudi2020global}. Finally, while modeling PSD-valued functions may be of interest {\em per se}, our model may also be used to apply convexity (or concavity) constraints on scalar-valued functions by imposing {\em equality} constraints on the Hessian (compared to the inequality constraints proposed by \cite{aubin2021handling}), which again allows leveraging interpolation inequalities.

\paragraph{Contributions.}

 In the first part of this work, we propose a linear model for smooth PSD-valued functions that extends the SoS model of \cite{marteau2020non}, further studied by \cite{rudi2020global}. More precisely:
 
\begin{enumerate}
    \item We introduce a linear sum-of-squares (SoS) model for functions whose values are PSD matrices (\Cref{prop:lin_and_pos}). This model is linearly parameterized by a PSD operator in a Hilbert space, and is defined in terms of an arbitrary vector-valued feature map. In the main body of this work, this feature map is based on the canonical feature map of a scalar-valued RKHS ; hence our results do not require prior knowledge of vector-valued RKHSs. Then, for the sake of completeness, we introduce vector-valued RKHSs (vRKHSs) in \Cref{sec:vector_rkhs}, and provide proofs and statements of our results in the more general setting of vRKHSs in \Cref{sec:PSD_proofs,sec:cvx_proofs}.
    \item We provide a representer theorem (\Cref{thm:rep_thm}), that generalizes that of \cite{marteau2020non}  along with a finite-dimensional formulation for our model that only depends on training points, and derive the dual form of generic convex PSD-valued function learning problems (\Cref{thm:dual_formulation}). 
    \item We show in \Cref{thm:univ_approx} that the class of function we introduce is a universal approximator for PSD-valued functions, which justifies its use in a wide range of problems involving PSD constraints.
    \item Extending the results of \cite{rudi2020global} to the PSD setting, we provide scattered data inequalities for the eigenvalues of the PSD matrices output by PSD-valued kernel SoS functions in \Cref{thm:scattered_inequalities}. These bounds may be used to quantify constraint violations in the case where our model is used to subsample PSD constraints.
\end{enumerate}

In the second part of this paper, we apply our model to the problem of fitting a function under convexity constraints:

    \begin{enumerate}\setcounter{enumi}{4}
        \item In \Cref{thm:cvx_hess_function}, we show that the Hessian of any strongly convex function may be represented as a kernel SoS, which allows handling learning problems with convexity constraints using SoS equality constraints. 
        \item Leveraging \Cref{thm:scattered_inequalities}, we bound in \Cref{cor:scattered_convexity} the strong convexity constant of functions learned with subsampled SoS constraints, and bound the error induced by subsampling PSD constraints when optimizing a Lipschitz functional over smooth convex functions.
        \item Finally, we illustrate the SoS model on a convex regression task.
    \end{enumerate}
    
An implementation of our models is available at \url{https://github.com/BorisMuzellec/ConvexSoS}.

%% file: sections/background.tex
\section{Background and Setting}\label{sec:background}
We start by stating notation and recalling background material on reproducing kernel Hilbert spaces and positive-semidefinite operators.

\subsection{Notation}

For $n \in \NN$, $[n]$ denotes the set $\{1, 2, \dots, n\}$. For two Hilbert spaces $\hh$ and $\Gcal$  $\Lcal(\hh, \Gcal)$ denotes the set of bounded linear maps from $\hh$ to $\Gcal$, and we write $\Lcal(\hh) = \Lcal(\hh, \hh)$ for short. $\Ccal(\Xcal, \Ycal)$ denotes the set of continuous maps from $\Xcal$ to $\Ycal$, and we write $\Ccal(\Xcal) = \Ccal(\Xcal, \RR)$ for short. Likewise, for $m \in \NN_+$, $\Ccal^m(\Xcal, \Ycal)$ (resp.\ $\Ccal^m(\Xcal)$) denotes the set of $m$ times continuously differentiable functions from $\Xcal$ to $\Ycal$ (resp.\ $\Xcal$ to $\RR$). If $\bA$ is a symmetric matrix, $\la_{\max}(\bA)$ (resp.\ $\la_{\min}(\bA)$) denotes its largest (resp.\ smallest) eigenvalue.
For a function $f$ of $\dimone$ variables and a multi-index $\alpha = (\alpha_1, ..., \alpha_\dimone) \in \NN^\dimone$, we denote the corresponding partial derivative of $f$ (when it exists)
$$
\partial^\alpha f(x) = \frac{\partial^{|\alpha|}}{\partial x_1^{\alpha_1}\cdots\partial x_\dimone^{\alpha_\dimone}} f(x),
$$
where $| \alpha| \defeq \sum_{i=1}^\dimone \alpha_i$. Following \cite{rudi2020global}, for a function $u$ that is $m$ times differentiable on $\Xcal$, we define the following semi-norm: 
\begin{align}\label{eq:def_semi_norm}
|u|_{\Xcal, m} \defeq \underset{|\alpha| = m}{\max}~\underset{x \in \Xcal}{\sup}|\partial^\alpha u(x)|.
\end{align}

\subsection{Positive-Semidefinite Matrices and Operators}
Let $\hh$ be a Hilbert space with inner product $\dotp{\cdot}{\cdot}_\hh$. For a linear operator $A :\hh \rightarrow \hh$, we denote $A^*$ its adjoint, $\tr(A)$ its trace and $\|A\|_\HS = \sqrt{\tr (A^*A)}$ its Hilbert-Schmidt norm. When $\hh$ is finite-dimensional, $A$ can be identified with its matrix representation, in which case the Hilbert-Schmidt norm is equal to the usual Frobenius norm $\|\cdot\|_F$. We denote the set of bounded self-adjoint operators on $\hh$ by $\sym{\hh}$, in finite dimension this corresponds to the set of symmetric matrices. We say that $A$ positive semi-definite (PSD) and write $A \succeq 0$ if $A$ is bounded, self-adjoint and if for all $x \in \hh$, it holds $\dotp{x}{Ax} \geq 0$. If further it holds $\dotp{x}{Ax} = 0 \implies x = 0$, we say that $A$ is positive-definite (PD) and write $A \succ 0$. We respectively denote $\pdm{\hh}$ and $\pd{\hh}$ the set of PSD and PD operators on $\hh$. For a symmetric matrix with eigenvalue decomposition $\bM = \bU \Sigma \bU^T$, we define its positive and negative parts as $[\bM]_+ = \bU \max(0, \Sigma) \bU^T$ and $[\bM]_- = \max(0, -\Sigma)\bU^T$ (where the $\max$ is applied elementwise).  

\subsection{Kernels and Reproducing Kernel Hilbert Spaces (RKHS).}
We refer to \citet{steinwart2008support} and \citet{paulsen2016introduction} for a more complete covering of the subject. Let $\Xcal$ be a set and $k : \Xcal \times \Xcal \mapsto \RR$. $k$ is a positive-definite kernel if and only if for any set of points $x_1, ..., x_n \in \Xcal$, the matrix of pairwise evaluations $K_{ij} = k(x_i, x_j), i, j \in [n]$ is positive semi-definite.
Given a kernel $k$, there exists a unique associated reproducing kernel Hilbert space (RKHS) $\hh$, that is, a Hilbert space of functions from $\Xcal$ to $\RR$ satisfying the two following properties:
\begin{itemize}
    \item For all $x\in \Xcal$, $k_x \defeq k(x, \cdot) \in \hh$;
    \item For all $f \in \hh$ and $x\in \Xcal$, $f(x) = \dotp{f}{k_x}_\hh$. In particular, for all $x, x' \in \Xcal$ it holds $\dotp{k_x}{k_{x'}}_\hh = k(x, x')$.
\end{itemize}
A feature map is a bounded map $\phi : \Xcal \mapsto \hh$ such that $\forall x, x' \in \Xcal, \dotp{\phi(x)}{\phi(x')}_\hh = k(x, x')$. A particular instance is $x \mapsto k_x$, which is referred to as the canonical feature map.

A key advantage of RKHSs is that they may be used to translate function-fitting problems into finite-dimensional problems. Indeed, consider a minimization problem of the form
\begin{align}\label{eq:generic_rkhs_pb}
    \underset{f \in \Hcal}{\min}~~L(f(x_1), ..., f(x_n)) + \Omega(\|f\|_\hh),
\end{align}
where $\Omega : \RR_+ \mapsto \RR$ is a strictly increasing function. In that case, the representer theorem \citep[see, e.g.,][and references therein]{steinwart2008support,paulsen2016introduction} shows that solutions of $\eqref{eq:generic_rkhs_pb}$ admit the following finite representation: $f = \sum_{i=1}^n \alpha_i k_{x_i}$ for some $\alpha \in \RR^n$.

%% file: sections/positive_operators.tex
\section{Learning Positive-Operator-Valued Functions with Kernel Sums-of-Squares}\label{sec:pos_op}

In this section, we introduce a non-parametric model for functions that take values in $\pdm{\RR^\dimtwo}$, the space of $\dimtwo\times \dimtwo$ positive semi-definite matrices.\footnote{In particular, the dimension $\dimone$ of the inputs need not match that of the output matrices, $\dimtwo$.} This model has a simple expression in terms of the feature map $\phi$ of a RKHS of scalar-valued functions $\hh$. Let $\hh^\dimtwo = \{(f_1, \dots, f_\dimtwo) : f_i \in \hh, i \in [\dimtwo] \}$, and let $\pdm{\hh^\dimtwo}$ denote the space of positive semi-definite self-adjoint operators from $\hh^\dimtwo$ to $\hh^\dimtwo$.
Given a feature map $\phi:\Xcal \to \hh$, we define our model as 
\begin{equation}\label{eq:model_def_formal}
F_A(x) = \Phi(x)^*A \Phi(x),~~\text{with}~~A \in \pdm{\hh^\dimtwo},
\end{equation}
where $\Phi(x) \in \Lcal(\RR^\dimtwo, \hh^\dimtwo)$ and the positive semidefinite operator $A \in \pdm{\hh^\dimtwo}$ correspond to
\begin{equation}\label{eq:def_phi}
\Phi(x) = \underbrace{\begin{pmatrix} 
\phi(x) & 0_\hh & \cdots & 0_\hh\\
0_\hh & \phi(x) & \cdots & 0_\hh\\
\vdots &  \vdots & \ddots & \vdots\\
0_\hh & 0_\hh & \cdots & \phi(x)
\end{pmatrix}}_{\displaystyle \dimtwo}, \qquad A = \begin{pmatrix} 
A_{11} & \dots & A_{1\dimtwo} \\
\vdots & \ddots & \vdots\\
A_{\dimtwo 1} & \dots & A_{\dimtwo\dimtwo}
\end{pmatrix}, \quad A \succeq 0,
\end{equation}
with $A_{ij} \in \Lcal(\hh),$ $i, j \in [\dimtwo]$. Formally, for $x\in \Xcal$ and $v \in \RR^d$, $\Phi$ is defined by $\Phi(x)v= (v_1\phi(x), \dots v_d \phi(x))$. Note that \cref{eq:model_def_formal} can be written equivalently as
\begin{align}\label{eq:model_def_informal}
F_A(x) = \begin{pmatrix} 
\dotp{\phi(x)}{ A_{11}\phi(x)} & \dots & \dotp{\phi(x)}{ A_{1\dimtwo}\phi(x)} \\
\vdots & \ddots & \vdots\\
\dotp{\phi(x)}{ A_{\dimtwo 1}\phi(x)} & \dots & \dotp{\phi(x)}{ A_{\dimtwo\dimtwo}\phi(x)}
\end{pmatrix} \in \RR^{\dimtwo\times \dimtwo}, ~~\text{with}~~A  \succeq 0.
\end{align}
for any $x \in \Xcal$. One can easily check that that the model in \cref{eq:model_def_formal} is a  generalization of \citeauthor{marteau2020non}'s model for non-negative scalar functions to PSD-valued functions. As such, we will show that they enjoy many similar properties, starting with linearity w.r.t. the parameter $A$, and the fact that, by construction, the model $F_A(x)$ outputs a positive semidefinite matrix for any $x \in \Xcal$, since $A$ is positive semidefinite. The proof can be found in \Cref{proof:lin_and_pos}, \cpageref{proof:lin_and_pos}.
\begin{proposition}[Linearity and pointwise positivity]\label{prop:lin_and_pos}
Given $A, B \in \Lcal(\hh^\dimtwo)$ and $\alpha, \beta \in \RR$ it holds $F_{\alpha A + \beta B}(x) = \alpha F_A(x) + \beta F_B(x)$. Moreover, if $A\succeq 0,$ $F_A(x) \in \pdm{\RR^\dimtwo}, \forall x \in \Xcal.$
\end{proposition}
\begin{remark}[Definition as a sum-of-squares]
Following the nomenclature of \cite{rudi2020global}, we refer to the model in \cref{eq:model_def_formal} as a kernel sum-of-squares (SoS). This denomination is motivated by the following observation: let $A = \sum_{k \geq 0} \sigma_k u_k \otimes u_k$ denote the eigenvector decomposition of $A \in \pdm{\hh^d}$, where $\sigma_k \geq 0$ and $u_k \in \hh^d, k \in \NN$. Then, by the reproducing property, we have $F_A(x) = \sum_{k \geq 0} \sigma_k u_k(x)u_k(x)^T, x \in \Xcal$. Hence, $F_A(x)$ is an infinite sum of ``squares'' of the form $h_k(x) h_k(x)^T$, with functions $h_k$ belonging to the product of RKHS $\hh^d$.
\end{remark}
\begin{remark}[Matrix-valued kernels]
The construction of the semi-definite model \eqref{eq:model_def_formal} holds for more general choices of output spaces and of feature maps $\Phi$. In particular, we can model a function whose values are operators on a separable Hilbert space $\cal{Y}$ beyond $\mathbb{R}^d$, by using any operator-valued feature map $\Phi: \Xcal \to \Lcal(\cal{Y}, \hh \otimes {\cal Y})$ and an operator $A \in \pdm{\hh \otimes {\cal Y}}$, where $\otimes$ is the tensor product. Examples of such feature maps include feature maps associated to vector-valued reproducing kernel Hilbert spaces (vRKHS), which generalize RKHSs to vector-valued functions \citep[][see]{carmeli2010vector,micchelli2006universal}. Like RKHSs, vRKHSs also have an associated kernel function $K$, which takes values in the space of operators $\Lcal({\cal Y}, {\cal Y})$, and satisfies $K(x, x') = \Phi(x)^*\Phi(x')$. As an example, the feature map introduced in \cref{eq:def_phi} leads to the simple matrix-valued kernel $K(x,x') = k(x,x')\eye_d$. In the remainder of the paper, we will only consider this particular setting, but we extend our results to the general matrix-valued kernel setting in the appendix. We introduce vRKHSs in more detail in \Cref{sec:vector_rkhs}.
\end{remark}

\subsection{Representer Theorem, Finite-Dimensional Representation, and Duality}\label{subsec:rep_results}
In the previous section, we introduced a model for matrix-valued functions \eqref{eq:model_def_formal} that takes values on the PSD cone (\Cref{prop:lin_and_pos}) and is linear with respect to its parameter. However, this parameter is described as an infinite-dimensional positive operator, which is difficult to handle in practice. In this section, we show how to obtain finite-dimensional equivalents of \eqref{eq:model_def_formal} when working on problems for problems of the form \eqref{eq:gen_sampled_pb} that involve a finite number of evaluations of $F_A$ at some given points $\{x_1, ..., x_n\} \subset \Xcal$. More precisely, we start by proving a representer theorem that allows expressing $A$ using a finite-dimensional matrix (\Cref{thm:rep_thm}), which generalizes Theorem 1 of \citet{marteau2020non}. Then, as an alternative to the infinite-dimensional features $\Phi(x_1), ..., \Phi(x_n)$, we derive finite-dimensional features $\Psi_1, ..., \Psi_n$ (\Cref{subsubsec:finite_rep}). Finally, we provide a dual formulation for problems of the form \eqref{eq:gen_sampled_pb} that reduces the number of parameters to optimize by a factor $n$ (\Cref{thm:dual_formulation}).

\subsubsection{Representer theorem.} Following \cite{marteau2020non}, for regularized problems of the form 
\begin{align}\label{eq:gen_sampled_pb}
    \min_{A \in \pdm{\hh^d}} \ L(F_A(x_1), \dots, F_A(x_n)) + \Omega(A),
\end{align}
with
\begin{equation}\label{eq:reg_term}
\Omega(A) = \lambda_1 \tr A + \frac{\lambda_2}{2} \|A\|^2_{HS}, ~~\lambda_1, \la_2 \geq 0, 
\end{equation}
we expect that optimal solutions admit a finite-dimensional representation. This is shown in the following theorem, proven in \Cref{proof:rep_thm}, \cpageref{proof:rep_thm}. Note that \Cref{thm:rep_thm} may be extended to a more general class of regularizers based on spectral functions.
\begin{theorem}[Representer theorem]\label{thm:rep_thm}
Let $L:\pdm{\RR^d}^n \rightarrow \RR\cup\{+\infty\}$ be lower semi-continuous and bounded below. Let further $\Omega$ be as in \cref{eq:reg_term} with $\la_1 > 0$ or $\la_2 > 0$. Then \cref{eq:gen_sampled_pb} has a solution $A_\star$ which may be written
\begin{equation}
A_\star = \sum_{i,j = 1}^n \Phi(x_i)\bC_{ij}\Phi(x_j)^*,
\end{equation}
with $\bC_{ij} \in \RR^{d\times d}, i, j \in [n]$, and $\bC = \left(\begin{smallmatrix} \bC_{11} & \dots & \bC_{1n}\\ \vdots & & \vdots \\  \bC_{n 1} & \dots & \bC_{nn} \end{smallmatrix}\right) \in \pdm{\RR^{nd}}$.  
Then, $F_{A_\star}$ can be equivalently written in a finite form in terms of $\bC_{ij}$, as follows,
\begin{equation}\label{eq:discrete_form}
    F_C(x) = \sum_{i,j = 1}^n k(x_i, x)k(x_j, x)\bC_{ij}, ~~x \in \Xcal.
\end{equation} 
If further $L$ is convex and $\lambda_2 > 0$, $A_\star$ is unique.
\end{theorem}

\subsubsection{Finite-dimensional representation.}\label{subsubsec:finite_rep}

The representation of \Cref{thm:rep_thm} can be further simplified. Indeed, let $\bK \in \RR^{n\times n}, K_{ij} = k(x_i, x_j), i, j \in [n]$
and
$\tilde{\bK} = \bK \otimes \eye_d \in \RR^{nd \times nd}$, where $\otimes$ denotes the Kronecker product.
By cyclical invariance of the trace and using the fact that $\Phi(x_j)^\star\Phi(x_i) = k(x_i, x_j) \eye_d$, we have
\begin{align}
\tr(A_\star)  = \sum_{i,j=1}^n \tr(\bC_{ij} \Phi(x_j)^* \Phi(x_i))
= \sum_{i,j=1}^n k(x_i, x_j) \tr\bC_{ij}
= \tr (\tilde{\bK}\bC),
\end{align}
and for $\ell \in [n]$,
\begin{align}
\begin{split}
F_{A_\star}(x_\ell) &= \sum_{i,j=1}^n k(x_i, x_\ell) k(x_j, x_\ell) \bC_{ij} 
=  (\tilde{\bK}\bC\tilde{\bK})_{\ell, \ell},
\end{split}
\end{align}
where $(\tilde{\bK}\bC\tilde{\bK})_{\ell, \ell} \in \RR^{d \times d}$ denotes the $\ell$-th $d \times d$ block matrix on the diagonal of $\tilde{\bK}\bC\tilde{\bK}$. In particular, if $\bR \in \RR^{n \times n}$ denotes the upper-triangular Cholesky factor of $\bK$ (i.e.\ $\bR$ is the upper triangular matrix satisfying $\bR^T\bR = \bK$), then $\tilde{\bR} = \bR \otimes \eye_d$ is the upper-triangular factor of $\tilde{\bK}$. Letting $\bB = \tilde{\bR} \bC \tilde{\bR}^T$, we then have
$$\tr A_\star = \tr \bB, \qquad F_{A_\star}(x_\ell) = \Psi_\ell^T \bB \Psi_\ell, ~~ \ell \in [n],$$
where $\Psi_\ell = \bR_{:, \ell} \otimes \eye_d \in \RR^{nd \times d}$ denotes the $\ell$-th $nd \times d$ block column of $\tilde{\bK}$ and $\bR_{:, \ell}$ is the $\ell$-th column of $\bR$. Further, the model $F_{A_\star}$ has the following finite dimensional representation
\begin{equation}\label{eq:finite_dimensional_form}
F_B(x) = \Psi(x)^T\bB\Psi(x),
\end{equation}
with 
$\Psi(x) = (\bR^{-T} v(x)) \otimes \eye_d \in \RR^{nd \times d}$, and $v(x) = (k(x, x_i))_{i=1}^n \in \RR^n$. 
To conclude, using \Cref{thm:rep_thm} and the change of variables above, problems of the form of \cref{eq:gen_sampled_pb} admit a fully finite-dimensional characterization in terms of the following optimization problem over a positive semi-definite matrix of size $nd \times nd$: 
\begin{align}\label{eq:B_problem}
\min_{\bB \in \pdm{\RR^{nd}}} L(F_B(x_1), ..., F_B(x_n)) + \Omega(\bB),
\end{align}
with $F_B(x_i) = \Psi_i^T\bB\Psi_i, ~i \in [n]$. 

\subsubsection{Dual formulation.}
When the loss function $L$ in \cref{eq:gen_sampled_pb} is convex, we may obtain an equivalent dual formulation that involves $n$ matrices of size $d\times d$, involving the Fenchel conjugate of $L$, defined as $L^*(Y) \defeq \underset{X \in \sym{\RR^d}^n}{\sup}\{ \sum_{i=1}^n \dotp{\bY^{(i)}}{\bX^{(i)}} - L(X)\},~~ Y \in \sym{\RR^d}^n $. 
\begin{theorem}\label{thm:dual_formulation}
Let $L : \sym{\RR^d}^n \rightarrow \RR$ be convex, l.s.c.\ and bounded from below. Let $\Omega$ be as in \cref{eq:reg_term}. Assume that \cref{eq:B_problem} admits a feasible point, and denote $L^*$ the Fenchel conjugate of $L$. Then,
\begin{enumerate}[(i)]
    \item If $\lambda_2 > 0$ and $\lambda_1 \geq 0$,  \cref{eq:B_problem} has the following dual formulation:
    \begin{align}\label{eq:dual_neg_part}
\sup_{\Gamma\in \sym{\RR^d}^n} -L^\star(\Gamma) - \frac{1}{2\la_2} \bigg\| \Big[\sum_{i=1}^n \Psi_i \Gamma^{(i)} \Psi_i^T + \lambda_1 \eye_{nd} \Big]_{-} \bigg\|_F^2,
\end{align}
where $\Gamma^{(i)} \in \sym{\RR^d}, i \in [n]$ denotes the $i$-th matrix element of $\Gamma \in \sym{\RR^d}^n$, and $[\bM]_{-}$ denotes the negative part of $\bM$. This supremum is attained. Further, if $\Gamma_\star$ is a optimal solution of \cref{eq:dual_neg_part}, an optimal $\bB^\star$ for \cref{eq:B_problem} is $$\bB^\star = \frac{1}{\la_2}\Psi\inv\Big[\sum_{i=1}^n \Psi_i \Gamma_\star^{(i)} \Psi_i^T + \lambda_1 \eye_{nd} \Big]_{-} \Psi^{-T}.$$
    \item If $\lambda_2 = 0$ and $\lambda_1 > 0$, \cref{eq:B_problem} has the following dual formulation:
    \begin{align}
\sup_{ \Gamma\in \sym{\RR^d}^n} -L^\star(\Gamma) ~~\st~~ \sum_{i=1}^n \Psi_i \Gamma^{(i)} \Psi_i^T + \lambda \eye_{nd} \succeq 0.
\end{align}
\end{enumerate}
\end{theorem}
Proof in \Cref{proof:dual_formulation}, \cpageref{proof:dual_formulation}.
\subsection{Approximation Properties}

In the previous sections, we have introduced a model for smooth PSD-valued functions that relies on infinite-dimensional operators and features, and then showed that for finite optimization problems of the form \eqref{eq:gen_sampled_pb} such problems admit a finite representation. We now study the approximation properties of \cref{eq:model_def_formal}. We start by showing that, under mild assumptions on the RKHS $\hh$, our model \eqref{eq:model_def_formal} is a universal approximator for PSD-valued functions (\Cref{thm:univ_approx}). 

We adapt the notion of \emph{universality}~\citep{micchelli2006universal} to PSD-valued functions: we say that a set of PSD-valued functions $\Fcal$ is universal if for any compact set $\Zcal \subset \Xcal$, any $\varepsilon > 0$ and any function $g \in \Ccal(\Zcal, \pdm{\RR^d})$, there exists $f \in \Fcal$ such that $\|f_{|\Zcal} - g\|_{\Zcal} \leq \varepsilon$, where $\|g\|_{\Zcal} \defeq \sup_{x \in \Zcal} \|g(x)\|_\infty$ denotes the max norm over $\Zcal$. In the following theorem, proved in \Cref{proof:univ_approx}, \cpageref{proof:univ_approx}, we show that our model for PSD-valued functions is universal.
\begin{theorem}[Universal approximation]\label{thm:univ_approx}
Let $d\in \NN, d\geq 1$, $\hh$ be a separable Hilbert space and $\phi: \Xcal \rightarrow \hh$ a universal map~\citep[see][]{micchelli2006universal}. Then 
\begin{equation}
    \Fcal \defeq \{F_A : A \in \pdm{\hh^d}, \tr A < \infty\}
\end{equation} 
is a universal approximator of $d\times d$ PSD-valued functions over $\Xcal$.
\end{theorem}
As examples, universal kernels include the Gaussian kernel $\exp(-\|x-x'\|^2/\sigma^2)$ or the exponential kernel $\exp(-\|x-x'\|/\sigma)$. Informally, \Cref{thm:univ_approx} shows that the SoS model \eqref{eq:model_def_formal} may be used to approximate any PSD-valued function arbitrarily well. As a consequence, this model is well-suited to PSD function learning problems, or to model problems with PSD constraints.

\subsection{Numerical Illustration: PSD Least-Squares Regression}
Let $x_1, ..., x_n \in \RR^{\dimone}$ and $\bM_1, ..., \bM_n \in \pdm{\RR^{\dimtwo}}$.
We wish to solve
\begin{align}\label{eq:PSD_regressions_primal}
    \Pcal \defeq \min_{\bB \in \pdm{\RR^{n\dimtwo}}} \frac{1}{2n}\sum_{i=1^n}\|F_\bB(x_i) - \bM_i\|_F^2 + \lambda_1 \tr \bB + \frac{\lambda_2}{2} \|\bB\|^2_F,
\end{align}
with $\la_2 > 0$ and $\la_1 \geq 0$.
\begin{figure}
    \centering
    \includegraphics[width=\textwidth]{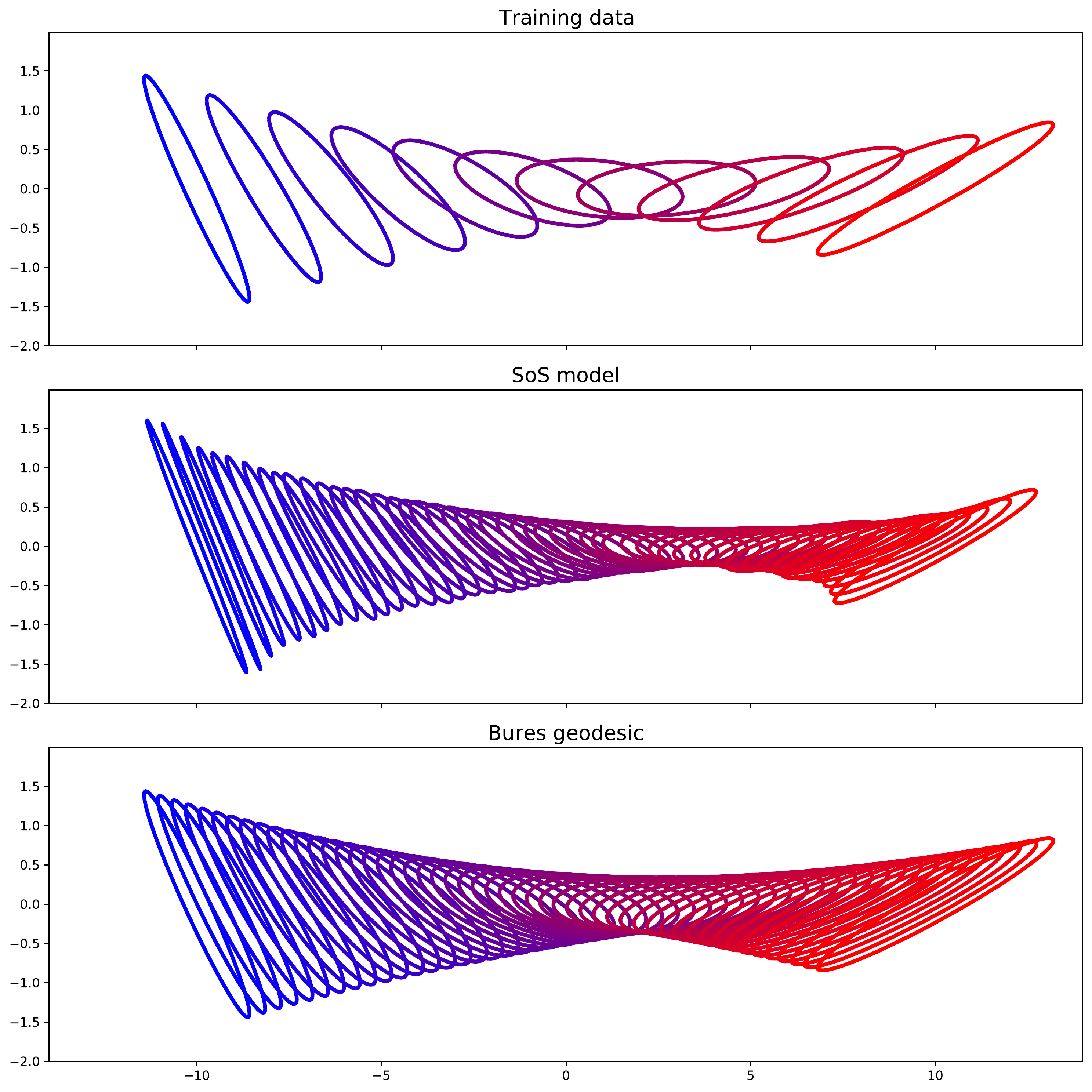}
    \caption{Interpolation of a Bures geodesic from $12$ data points (top). The matrices are represented as the level sets of Gaussian distributions: $\{x : \Ncal(x; 0, \Sigma) \leq r\}$ for $r=0.1$. We use the exponential kernel and select all hyperparameters using cross-validation. The learned model is represented in the middle figure, and the full geodesic is plotted in the bottom figure for comparison.}
    \label{fig:bures_geodesic}
\end{figure}
Following \Cref{thm:dual_formulation}, we have the following dual formulation:
\begin{align}\label{eq:PSD_regressions_dual}
    \Pcal = - \frac{1}{2} \min_{\Gamma \in \sym{\RR^\dimtwo}^n} n\sum_{i=1}^n\|\Gamma^{(i)}\|_F^2 + \sum_{i=1}^n \dotp{\Gamma^{(i)}}{\bM_i}_F + \frac{1}{\la_2} \bigg\| \Big[\sum_{i=1}^n \Psi_i \Gamma^{(i)} \Psi_i^T + \lambda_1 \eye_{n\dimtwo}\Big]_{-} \bigg\|_F^2,
\end{align}
with the primal-dual optimality relation $\bB^\star = \frac{1}{\la_2}\Psi\inv[\sum_{i=1}^n \Psi_i \Gamma_\star^{(i)} \Psi_i^T + \lambda_1 \eye_{nd}]_{-} \Psi^{-T}.$ As can be seen in this example, the dual formulation \eqref{eq:PSD_regressions_dual} has several advantages over the primal \eqref{eq:PSD_regressions_primal}. First, the dual problem has only $O(\dimtwo^2 n)$ parameters (and thus $O(\dimtwo^2 n)$ memory footprint), compared to the primal which has $O(\dimtwo^2 n^2 )$ parameters. Second, and more importantly, the dual problem no longer has PSD constraints, but only symmetry constraints, which makes it more amenable to (accelerated) first-order methods — at the price of a negative part term that must be computed using SVD. Since \eqref{eq:PSD_regressions_dual} is a smooth and strongly convex minimization problem (with convexity parameter $\mu = n$ and smoothness $L = n + \frac{1}{\la_2} \lambda_{\max}(\bK \circ \bK)$ with $\bK \in \RR^{n\times n}, K_{ij} = k(x_i, x_j), i, j \in [n]$), we may compute a solution using accelerated gradient descent~\citep{nesterov2003introductory}. In particular, since the gradient of the objective function in \eqref{eq:PSD_regressions_dual} is a vector of symmetric matrices whenever evaluated at symmetric arguments, we may solve \eqref{eq:PSD_regressions_dual} without need for projections if the $\Gamma^{(i)} \in \RR^{\dimtwo\times \dimtwo}, i \in [n]$ are initialized as symmetric matrices.

We illustrate our model on a geodesic interpolation task. Here, the matrices $\bM_1, ..., \bM_n \in \pdm{\RR^2}$ are sampled from a Bures geodesic~\citep{bhatia2019bures} joining $\bM_1$ to $\bM_n$, at evenly sampled times $x_1=0 \leq x_2 \leq ... \leq x_n = 1$. We fit a PSD kernel SoS model by solving \eqref{eq:PSD_regressions_dual} using the exponential kernel, and selecting the kernel bandwidth and regularization parameters $\la_1$ and $\la_2$ using leave-one-out cross-validation. 
The results are illustrated in \Cref{fig:bures_geodesic}, and in \Cref{fig:bures_geodesic_rank_1} in \Cref{sec:add_exp} on a rank one geodesic.
As can be seen in \Cref{fig:bures_geodesic}, the PSD model is able to faithfully learn the geodesic. Further, \Cref{fig:bures_geodesic_rank_1} shows that our model is able to handle singular inputs, in which case it yields matrices which have a larger conditioning, although we can observe a smoothing effect which yields rank 2 matrices instead of rank 1 in the true geodesic.

As an alternative to our model, one could consider learning PSD matrices under the form $L(x) L(x)^T$ where $L$ is an unconstrained matrix-valued function, or under the form $e^{L(x)}$. However, in the case of the least-squares objective \eqref{eq:PSD_regressions_primal}, both would result in non-convex problems, for which obtaining a global minimum is generally NP-hard.

Finally, further applications involving PSD-valued functions could be treated in a similar way, such as estimating the conditional covariance of an heteroscedastic Gaussian process, or learning the metric tensor of a Riemannian manifold.

\section{Using the Model to Approximate PSD Contraints in Optimization}\label{sec:approximating-constraints}
\cite{rudi2020global}  propose a general technique to deal with constrained optimization problems of the form $\min_{\theta} L(\theta)$ subject to $g(\theta, x) \geq 0, ~ \forall x \in \Xcal$, for an objective function $L$ and constraint function $g$. It consists in two main steps:
\begin{enumerate}
    \item First, express the dense set of inequality constraints as an equality constraint in terms of the PSD model for non-negative functions \citep{marteau2020non};
    \item Then, apply the equality constraints only on a suitably chosen set of points (e.g., if $\Xcal$ is compact, sampled uniformly at random).
\end{enumerate}
Given $\hat{X} = \{x_1, \dots, x_n\} \in \Xcal$, \citeauthor{rudi2020global} study the effect of approximating the problem above with the problem
\begin{align}\label{eq:discretized-problem}
\min_{\theta, \bB \in \pdm{\mathbb{R}^n} } L(\theta) + \Omega(\bB) ~~\text{subject to}~~ g(\theta, x_i) = \psi_i^\top \bB \psi_i, ~~i=1,\dots,n,
\end{align}
for suitable vectors $\psi_i \in \RR^n$, and $\Omega$ a regularization for $B$. They show that, this way, (a) it is possible to use results from approximation theory and leverage the degree of differentiability of the constraint function $g$, dramatically improving the convergence rate of the approximation for highly differentiable constraints; (b) if $L$ is convex and $g$ is linear in $\theta$, the resulting problem is a finite-dimensional convex problem.

Here we propose to use a similar technique for optimization problems subject to positive semidefinite constraints, using the model introduced in \cref{eq:model_def_formal}, and in particular its finite-dimensional characterization \cref{eq:finite_dimensional_form}. In particular, given a problem of the form 
\begin{equation}\label{eq:g_theta_continuous}
\min_{\theta \in \Theta} L(\theta) \quad \textrm{subject to} \quad g(\theta, x) \succeq 0, ~~ \forall x \in \Xcal,
\end{equation}
for $g: \Theta \times \Xcal \to \sym{\mathbb{R}^d}$, we suggest to apply the same steps above on the positive semidefinite constraints using the model in \cref{eq:model_def_formal}. The first step corresponds to
$$\min_{\theta \in \Theta, A \succeq 0} L(\theta) + \Omega(A) \quad \textrm{subject to} \quad g(\theta, x) = F_A(x), ~~ \forall x \in \Xcal,$$
where the regularizer $\Omega$ is defined in \cref{eq:reg_term} and $F_A$ is the model in \cref{eq:model_def_formal}. The second step leads to the following problem:
\begin{equation}\label{eq:g_theta_sampled}
\min_{\theta \in \Theta, {\bf B} \in \pdm{\mathbb{R}^{nd}}} L(\theta) + \Omega(\bB) \quad \textrm{subject to} \quad g(\theta, x_i) = \Psi_i^\top {\bf B} \Psi_i, ~~  i=1,\dots,n,
\end{equation}
where $\Psi_i \in \mathbb{R}^{nd\times d}$ (defined in \cref{subsubsec:finite_rep}) are efficiently computable. Note that the latter problem has a finite number of constraints, is finite dimensional and, when $L$ is convex and $g$ is linear in $\theta$, can be solved with standard techniques of convex optimization. In the next theorem we quantify the effect of approximating the constraint set and show that we can leverage the degree of differentiability of the constraint function. Given $\theta$ and defining $F:\Xcal \to \sym{\mathbb{R}^d}$ as $F(x) = g(\theta, x)$, we study the error of approximating the constraint $F(x) \succeq 0, \forall x \in \Xcal$, with the constraint $F(x_i) = \Psi_i^T \bB \Psi_i, i \in [n]$ for some $\bB \in \pdm{\RR^{nd}}$. We quantify the violation of the constraint $F$ outside of $\hat{X}$, i.e., $F(x) \succeq - \varepsilon I$ for any $x \in \Xcal$ and for a $\varepsilon$ that depends on the smoothness of the kernel, the smoothness of $F$ and the {\em fill distance} of $\hat{X}$ w.r.t. $\Xcal$:
\begin{equation}\label{eq:hXX}
    h_{\hat{X}, \Xcal} \defeq \underset{x \in \Xcal}{\sup}~\underset{i \in [n]}{\min}~\|x-x_i\|.
\end{equation}
In particular, when $F$ is $m$-times differentiable, \Cref{thm:scattered_inequalities} shows that if $F(x_i) = \Psi_i^T \bB \Psi_i, i \in [n]$ then $F(x) \succeq - \varepsilon I$ over the whole $\Xcal$, where $\varepsilon = O(h_{\hat{X}, \Xcal}^m)$, leveraging the degree of differentiability of the function $F$. This result is will be used later in \Cref{sec:convex_functions} to study the optimization problem associated to approximating a convex function. 
Like \cite{rudi2020global}, we require the following assumptions on the domain $\Xcal$ and the RKHS $\hh$.

\begin{assumption}[Geometric properties of $\Xcal$]\label{ass:geom_domain}
There exists $r > 0$ and a bounded set $S \subset \RR^\dimone$ such that $\Xcal = \underset{x \in S}{\cup}B_r(x)$, where $B_r(x)$ is the open ball centered in $x$ of radius $r$.
\end{assumption}

\begin{assumption}[Properties of the RKHS $\hh$]\label{ass:rkhs}
Given a bounded open set $\Xcal \subset \RR^\dimone$, let $\hh$ be a RKHS of scalar functions on $\Xcal$ with norm $\|\cdot\|_\hh$ kernel $k$ satisfying the following conditions:
\begin{enumerate}[a)]
    \item\label{item:product_bounded} There exists $M \geq 1$ such that 
    $\|u \cdot v\|_\hh \leq M \|u\|_\hh \|v\|_\hh, u, v \in \hh;$
    \item\label{item:composition} $a\circ v \in \hh$ for any $a \in \Ccal^\infty(\RR^q),~~ v = (v_1, ..., v_q),~~ v_j \in ~~\hh, j \in [q]$;
    \item\label{item:kernel_derivatives_bounded} For some $m \in \NN_+$ and $D_m \geq 1$, the kernel $k$ satisfies
    $$ \underset{|\alpha| = m}{\max}~\underset{x,y\in \Xcal}{\sup}~|\partial_x^\alpha \partial_y^\alpha k(x, y)| \leq D_m^2 < \infty;$$
\end{enumerate}

\end{assumption}

Like \cite{rudi2020global}, let us remark that that \Cref{ass:rkhs}\ref{item:kernel_derivatives_bounded} requires that $\hh$ is a RKHS with a $m$ times differentiable kernel, and implies that $\hh \subset \Ccal^m(\Xcal)$ and $|u|_{\Xcal, m} \leq D_m \|u\|_\hh$. 
While the SoS model, introduced later in \eqref{eq:model_def_formal}, defines a valid PSD-valued function for any feature map $\Phi$, \Cref{ass:geom_domain} and \Cref{ass:rkhs} are required to obtain the sampled inequality bounds in \Cref{thm:scattered_inequalities} and \Cref{cor:scattered_convexity} and the PSD representation result for Hessians of smooth and strongly convex functions in \Cref{thm:cvx_hess_function}.

\begin{example}[Sobolev kernel~\citep{wendland2004scattered}]
Let $\dimone \in \NN_+$ and $\Xcal \subset \RR^\dimone$ be a bounded open set. Let $s > \dimone / 2$, and define
\begin{equation}\label{eq:sobolev_kernel}
    k_s(x, x') = c_s \|x-x'\|^{s - \dimone/2} \Kcal_{s-\dimone / 2}(\|x-x'\|), ~~\forall x, x' \in \Xcal,
\end{equation}
where $\Kcal_{s-\dimone / 2} : \RR \rightarrow \RR$ is the Bessel function of the second kind with parameter $s - \dimone / 2$ \cite[see e.g.][Definition 5.10]{wendland2004scattered}, and $c_s = \frac{2^{1 + \dimone/2 - s}}{\Gamma(s - \dimone / 2)}$ is chosen so that $k_s(x, x) = 1$ for all $x\in \Xcal$. Then, the function $k_s$ is a kernel. When $\Xcal$ has locally Lipschitz boundary (and in particular, when \Cref{ass:geom_domain} is satisfied), its corresponding RKHS $\hh$ is $W_2^S(\Omega)$, the Sobolev space of functions whose weak-derivatives are square-integrable up to order $s$~\citep{adams2003sobolev}. Then \cref{eq:sobolev_kernel} satisfies \Cref{ass:rkhs} for any $m\in \NN_+$ s.t.\ $m < s - \dimone/2$ \citep[see][Proposition 1]{rudi2020global} with the following constants:
\begin{equation}
    M = (2\pi)^{\dimone/2}2^{s + 1/2}, \quad D_m = (2\pi)^{\dimone/2}\sqrt{\frac{\Gamma(m + \dimone/2)\Gamma(s- \dimone/2 - m)}{\Gamma(s-\dimone/2)\Gamma(\dimone/2)}}.
\end{equation}
Finally, note that in the particular case $s = \dimone / 2 + 1/2$, we have $k_s(x, x') = \exp(-\|x-x'\|)$ and that a scale factor can be added as $k_s(x, x') = \exp(-\|x-x'\|/\sigma)$, and similarly in \cref{eq:sobolev_kernel} by replacing $\|x-x'\|$ with $\|x-x'\|/\sigma$.
\end{example}
\begin{theorem}\label{thm:scattered_inequalities}
Let $\Xcal \subset \RR^\dimone$ satisfy \Cref{ass:geom_domain} for some $r>0$. Let $k$ satisfy \Cref{ass:rkhs}\ref{item:product_bounded} and \ref{item:kernel_derivatives_bounded} for some $m \in \NN_+$. Let $\hat{X} = \{x_1, ..., x_n\} \subset \Xcal$ such that $h_{\hat{X}, \Xcal} \leq r \min(1, \frac{1}{18(m-1)^2})$,  $h_{\hat{X}, \Xcal} \defeq \underset{x \in \Xcal}{\sup}~\underset{i \in [n]}{\min}~\|x-x_i\|$.
Let $F \in C^m(\Xcal, \sym{\RR^\dimtwo})$ and assume there exists $\bB \in \pdm{\RR^{nd}}$ such that 
$$F(x_i) = \Psi_i^T \bB \Psi_i, ~~i \in [n],$$
where the $\Psi_i$'s are defined in \Cref{subsec:rep_results}. 
Then, it holds
\begin{equation}
    \forall x \in  \Xcal,~~ \la_{\min}(F(x)) \geq - \varepsilon, ~~\text{where}~~ \varepsilon = C h_{\hat{X},  \Xcal}^m,
\end{equation}
and $C = C_0(\sum_{i=1}^\dimtwo |F_{ii}|_{\Xcal, m} + M D_m \tr\bB)$ with $C_0 = 3\frac{\dimone^m}{m!}\max(1, 18(m-1)^2)^m$, and where $F_{ij}$ is the the scalar-valued function corresponding to the element of $F$ with index $i, j \in [\dimtwo]$.
\end{theorem}

As a remark, let us notice that the bounds in \Cref{thm:scattered_inequalities} involve both the dimension $\dimone$ of $\Xcal$ and the size $d$ of the matrices output by $F$, which need not be equal. 
\Cref{thm:scattered_inequalities} is proved in \Cref{proof:scattered_inequalities}, \cpageref{proof:scattered_inequalities}. We sketch here the main arguments of the proof. The followed strategy consists in using scattered data inequality for scalar functions given by Theorem 13 of \cite{rudi2020global} (itself adapted from \cite{wendland2004scattered}). To do so, we start by considering a fixed direction in the unit sphere $u \in S^{\dimtwo-1}$, and apply the theorem to the smooth scalar function $g_u(.) \defeq u^T (F(.) - F_\bB(.)) u$, to obtain lower bounds on $u^T F(.) u$. When then derive global bounds (independent of $u$) over $S^{\dimtwo-1}$. Finally, since the lowest eigenvalue of a symmetric matrix $\bM \in \sym{\RR^d}$ is $\underset{u\in S^{d-1}}{\min} u^T \bM u$, we obtain a global lower bound on $\lambda_{\min}(F(x))$ for $x \in \Xcal$. 
Using \Cref{thm:scattered_inequalities}, we may turn $F$ into a function that is PSD everywhere by adding $\varepsilon \eye_d$. Whenever the constraint function $g(\theta, x)$ allows relating the addition of the constant term $\varepsilon \eye_d$ to a change in $\theta$ (e.g.\ if $g(\theta, x) = A(x)[\theta] + B(x)$ and $A[x]$ linear), we may use \Cref{thm:scattered_inequalities} to bound $L(\theta_\star) - L(\hat{\theta})$, where $\theta_\star$ is the solution of \cref{eq:g_theta_continuous} and $\hat{\theta}$ the solution of \cref{eq:g_theta_sampled}, as in \Cref{thm:subsampling_error} in \Cref{sec:convex_functions} for optimization under convexity constraints.

%% file: sections/convex_functions.tex
\section{Modeling Convex Functions with Sum-of-Squares Hessians}\label{sec:convex_functions}
In \Cref{sec:pos_op}, we introduced a model for functions that take PSD matrices as values. While some problems such as those mentioned in the same section involve learning PSD-valued functions directly, positive semidefiniteness most commonly appears as a constraint in optimization problems. In particular, when the constraint is smooth enough, the prevalent problem of optimizing a scalar-valued function under convexity constraints can be reformulated using PSD constraints on the Hessian, which allows leveraging the smoothness of the constraint, as shown in \Cref{thm:scattered_inequalities}. That is, problems of the form
\begin{align}\label{eq:gen_cvx_pb}
\min J(f)~~\st~~ f \in \Ccal^2(\Xcal) ~\text{is convex},
\end{align}
may be written as follows
\begin{align}\label{eq:gen_cvx_pb_hessian}
\min_{f \in \Ccal^2(\Xcal)} J(f)~~\st~~ H_f(x) \succeq 0, ~~\forall x \in \Xcal.
\end{align}
Problems of the form \eqref{eq:gen_cvx_pb} are notoriously difficult to handle and as a consequence most works on this topic have focused on piecewise-affine functions~\citep{seijo2011nonparametric}, or functions defined on low-dimensional (usually 2D) domains~\citep{merigot2014handling}.  
In \Cref{sec:approximating-constraints}, we extended the approach introduced by \cite{rudi2020global} for non-negative constraints, to approximate optimization problems subject to a dense set of positive semidefinite constraints, as the problem in \cref{eq:gen_cvx_pb_hessian}. In the present section, we use the proposed extension to approximate efficiently \cref{eq:gen_cvx_pb_hessian}. In particular, in the next two subsections we will transform \cref{eq:gen_cvx_pb_hessian} as described in \Cref{sec:approximating-constraints}, and we will derive quantitative bounds on the approximation error and computational complexity of the finite dimensional problem in \cref{eq:cvx_pb_subsampled}. 

\subsection{From Positive Semidefinite to Equality Constraints}
The first step of the approach described in \cref{sec:approximating-constraints} consists in transforming the positive semidefinite constraints into equality constraints with respect to the PSD sum-of-squares model. In our case, this corresponds to leveraging the PSD sum-of-squares model introduced in \Cref{sec:pos_op} to enforce positiveness constraints on the Hessian of the variable $f$. After adding a regularizer, which is required to obtain the finite-dimensional representation of \Cref{thm:rep_thm} and to control the eigenvalues of $H_f$ (\Cref{thm:scattered_inequalities}), the resulting problem has the following form:
\begin{align}\label{eq:gen_sos_cvx_pb}
\min_{f \in \hh, A \in \pdm{\hh^d}} J(f) + \Omega(A) ~~\st~~ H_f(x) = F_A(x), ~~\forall x \in \Xcal,
\end{align}
where $\Omega$ is defined in \cref{eq:reg_term}. In the following theorem, we show that encoding the Hessian as a PSD SoS allows recovering any sufficiently smooth strongly convex function.
\begin{theorem}[PSD sum-of-squares representation for convex functions]\label{thm:cvx_hess_function}
Let $\Xcal \subset \RR^d$ and $\hh$ be a RKHS of functions on $\Xcal$ satisfying \Cref{ass:rkhs}\ref{item:product_bounded}-\ref{item:kernel_derivatives_bounded}, and such that $\Ccal^s(\Xcal) \subset \hh$, $s \in \NN$. Let $f \in \Ccal^{s + 2}(\Xcal)$ be strongly convex.  Then, there exists $A \in \pdm{\hh^d}$ such that $H_f(x) = F_A(x), \forall x \in \Xcal$.  
\end{theorem}
Proof in \Cref{proof:cvx_hess_function}, \cpageref{proof:cvx_hess_function}. Let us mention that the smoothness constraint $\Ccal^s \subset \hh$ implies $s > d/2$, but that the strong convexity is merely a sufficient condition. As an example, the function $f:(x, y) \in \RR^2 \rightarrow \frac{1}{2}(x - y)^2$ admits the constant Hessian $H_f(x, y) = \begin{psmallmatrix} 1 & -1 \\ -1 & 1 \end{psmallmatrix}$, which is positive semi-definite, but not positive definite. However, its Hessian may be encoded as a PSD SoS for any RKHS containing constant functions. Hence, we conjecture that finer existence conditions could be proven.

\subsection{Finite-Dimensional Representation}

In this section, we show how enforcing convexity via a SoS model for the Hessian, $H_f(x) = F_A(x),  ~\forall x \in \Xcal$ with $A \succeq 0$, leads to tractable optimization. 

\subsubsection{Subsampling the Constraints.}

The generic problem \eqref{eq:gen_sos_cvx_pb} has an infinite number of constraints. Hence, it is not amenable to computation. As described in \Cref{sec:approximating-constraints}, we subsample its constraints on a set $\{x_1, ..., x_n\} \subset \Xcal$. Following \Cref{thm:rep_thm}, we then obtain a finite-dimensional version of problem \eqref{eq:gen_sos_cvx_pb}:
\begin{align}\label{eq:cvx_pb_subsampled}
\min_{f \in \hh, \bB \in \pdm{\RR^{nd}}} J(f) + \Omega(\bB) ~~\st~~ H_f(x_j) = \Psi_j^T \bB \Psi_j,~~ j \in [n].
\end{align}
Subsampling the constraints guarantees that solutions of \cref{eq:cvx_pb_subsampled} are convex in the neighborhood of $\{x_1, ..., x_n\}$, but does not guarantee global convexity. However, in the next result we  leverage \Cref{thm:scattered_inequalities} to obtain global bounds on the convexity deficit of $f$, i.e., the magnitude of the quadratic function that we have to add to $f$ to make it convex, as a function of the fill distance $h_{\hat{X}, \Xcal} \defeq \underset{x \in \Xcal}{\sup}~\underset{i \in [n]}{\min}~\|x-x_i\|.$
\begin{corollary}[\Cref{thm:scattered_inequalities}]\label{cor:scattered_convexity}
Let $\Xcal$ satisfy \Cref{ass:geom_domain} for some $r>0$. Let $k$ satisfy \Cref{ass:rkhs}\ref{item:product_bounded} and \Cref{ass:rkhs}\ref{item:kernel_derivatives_bounded} for some $m > d/2$. Let $\hat{X} = \{x_1, ..., x_n\} \subset \Xcal$ such that $h_{\hat{X}, \Xcal} \leq r \min(1, \frac{1}{18(m-1)^2})$.
Let $f \in \Ccal^{m + 2}(\Xcal)$ and $\bB \in \pdm{\RR^{dn}}$ be such that $$H_f(x_j) = \Psi_j^T \bB \Psi_j,~~ j \in [n].$$ Then, 
$$f(x) + \tfrac{\eta}{2}\|x\|^2 ~~\textrm{is convex} \quad \textrm{for} \quad \eta \geq C h_{\hat{X},  \Xcal}^m,$$
where $C = C_0(\sum_{i=1}^\dimtwo |(H_{f})_{ii}|_{\Xcal, m} + M D_m \tr\bB)$ and $C_0$ is as in \Cref{thm:scattered_inequalities}, and where $(H_{f})_{ij}$ is the the scalar-valued function corresponding to the element of the Hessian $H_f$ with index $i, j \in [\dimtwo]$. 
\end{corollary}
Proof in \Cref{proof:scattered_convexity}, \cpageref{proof:scattered_convexity}. Corollary \ref{cor:scattered_convexity} guarantees that the strong convexity deficit of $f$ goes to zero as the the number of subsampled inequalities increases, provided they cover the domain $\Xcal$. It is important to note that the strong convexity deficit goes to zero with a rate that is exponentially decreasing in the degree of differentiability of $f $.
Let us remark that another option to obtain convex function beyond adding $\tfrac{\eta}{2}\|x\|^2$ correspond to enforcing the constraints $H_f(x_j) = \Psi_j^T \bB \Psi_j + \eta \eye_\dimtwo,~~ j\in [n]$ in \cref{eq:cvx_pb_subsampled}. However, the constant $\eta$ given by Corollary \ref{cor:scattered_convexity} depends on $f$ and $\bB$ and is therefore not known in advance. Under further assumptions on the functional $J$, $\eta$ can be used to quantify the error induced by solving \eqref{eq:cvx_pb_subsampled} instead of \eqref{eq:gen_cvx_pb_hessian}, as shown in the following theorem.

\begin{theorem}\label{thm:subsampling_error}
Let $\Xcal$ satisfy \Cref{ass:geom_domain} for some $r>0$. Let $\hat{X} = \{x_1, ..., x_n\} \subset \Xcal$ such that $h_{\hat{X}, \Xcal} \leq r \min(1, \frac{1}{18(m-1)^2})$. Let $\hh, k$ satisfy \Cref{ass:rkhs}\ref{item:product_bounded} and \Cref{ass:rkhs}\ref{item:kernel_derivatives_bounded} for some $m > d/2$, and let $\Fcal$ satisfy \Cref{ass:rkhs}\ref{item:kernel_derivatives_bounded} for $m' = m + 2$. Let $(f_\star, A_\star)$ be a solution of
\begin{align}\label{eq:continuous_pb_psd_constraints}
    \underset{\substack{f \in \Fcal,~A \in \pdm{\hh^d}}}{\min} J(f) ~~\st~~~ H_f(x) = F_A(x), ~~x \in \Xcal,
\end{align}
and $(\hat{f}, \hat{\bB})$ be the solution of
\begin{align}\label{eq:cvx_pb_subsampled_reg}
\min_{f \in \Fcal, ~\bB \in \pdm{\RR^{nd}}} J(f) + \rho \|f\|^2_\Fcal + \la \tr \bB ~~\st~~ H_f(x_j) = \Psi_j^T \bB \Psi_j,~~ j \in [n],
\end{align}
 and let $\eta$ be as in \Cref{cor:scattered_convexity}. Assume that $J$ is $L$-Lipschitz for some seminorm $N(\cdot)$. Then, $\tilde{f} = x \mapsto \hat{f}(x) + \tfrac{\eta}{2}\|x\|^2$ is a convex function on $\Xcal$ and it holds
\begin{align}\label{eq:subsampling_error}
     J(\tilde{f}) - J(f_\star) \leq R^2 (1 + \tfrac{M C_1}{\la} h^m_{\hat{X}, \Xcal}) \,\rho  ~+~ C_1 d R ~ h^m_{\hat{X}, \Xcal}.
\end{align}
where $R^2 = \frac{\la}{\rho}\tr A_\star +  \|f_\star\|_\Fcal^2$ and  $C_1 = \tfrac{1}{2} L C_N C_0 \max(D_m, D_{m+2})$, $C_N = N(s)$ and $s$ is the function $s(x) = \|x\|^2$ for $x \in \RR^d$. In particular, when $\rho = \la = M C_1 h^m_{\hat{X}, \Xcal}$, then
\begin{align}
     J(\tilde{f}) - J(f_\star) ~\leq~ Q ~ h^m_{\hat{X}, \Xcal}, 
\end{align}
with $Q = 4MC_1 (\tr A_\star +  \|f_\star\|_\Fcal^2 + \frac{d^2}{M^2})$.
\end{theorem}

Proof in \Cref{proof:subsampling_error}, \cpageref{proof:subsampling_error}. Note that if we choose the discretization points uniformly at random in $\Xcal$ and $\Xcal$ is a convex set, then $h_{\hat{X}, \Xcal} \leq (C\frac{1}{n}\log(1/\delta))^{1/d}$, with probability $1-\delta$. Then, the algorithm above leads to an error in the order of 
$$ J(\tilde{f}) - J(f_\star) \leq ~Q'~(\tr A_\star +  \|f_\star\|_\Fcal^2 + \frac{d^2}{M^2})~n^{-m/d},$$
where $Q' = Q (\log \frac{1}{\delta})^{m/d}$, while $Q$ depends only on $m, d$ and is exponential in $\max(m, d)$.

\subsubsection{Encoding the constraints.}
To enforce the constraints $H_f(x_j) = \Psi_j^T \bB \Psi_j,~ j \in [n]$  we must choose a representation of the Hessian of $f$ at points $\{x_1, ..., x_n\}$. In this section, we propose two such representations. The first (i) leverages the reproducing property on derivatives \citep{zhou2008derivative}, to obtain an exact representation, whereas the second (ii) is an approximate representation, that is only asymptotically accurate but is more amenable to computation.\\

\textit{(i) Reproducing property for derivatives.}
If the kernel $k$ is regular enough, reproducing properties hold for function derivatives~\citep{zhou2008derivative}. More precisely, if $k \in \Ccal^{2s}(\Xcal \times \Xcal)$ for some $s \in \NN_+$, then it holds:
\begin{itemize}
    \item $\forall x\in \Xcal$, $\forall \alpha \in \NN^{\dimone}~~\st~~|\alpha| \leq s$, $\partial^\alpha k_x \defeq \frac{\partial^{|\alpha|} k(s, \cdot)}{\partial s_1^{\alpha_1}, ..., \partial s_\dimone^{\alpha_\dimone}}\Bigr|_{s = x} \in \hh$;
    \item $\forall x \in \Xcal, \forall f \in \hh, \forall \alpha \in \NN^{\dimone}~~\st~~|\alpha| \leq s$, $\partial^\alpha f(x) = \dotp{f}{\partial^\alpha k_x}_\hh$. In particular, for all $x, x' \in \Xcal$ it holds $\dotp{\partial^\alpha k_x}{\partial^\alpha k_{x'}}_\hh = \frac{\partial^{2|\alpha|} k(s, t)}{\partial s_1^{\alpha_1}, ..., \partial s_\dimone^{\alpha_\dimone}\partial t_1^{\alpha_1}, ..., \partial t_\dimone^{\alpha_\dimone}} \Bigr|_{\substack{s=x\\t=x'}}$.
\end{itemize}
In this work, we will mostly consider derivatives up to order $2$. Hence, we will use the following simplified notations:
\begin{align}
\partial_i k_x \defeq \frac{\partial k(s, \cdot)}{\partial s_i}\Bigr|_{s = x}~~~\text{and}~~~\partial_{ij} k_x \defeq \frac{\partial^2 k(s, \cdot)}{\partial s_i \partial s_j}\Bigr|_{s = x},~~ x\in\Xcal, i,j \in [\dimone].
\end{align}
Let $\{e_p, p \in [\dimtwo]\}$ denote the canonical basis of $\RR^\dimtwo$. Then, for $v \in \Xcal$ we have $H_f(v) = \sum_{p, q = 1}^\dimtwo \frac{\partial^2f(v)}{\partial v_p v_q} e_p e_q^T$. Hence, we may rewrite the constraints as
\begin{equation}\label{eq:exact_rep}
\sum_{p, q=1}^d \dotp{f}{\partial_{pq}k_{x_j}} e_p e_q^T = \Psi_j^T \bB \Psi_j, ~~j \in [n].
\end{equation}
This yields the following problem:
\begin{equation}\label{eq:cvx_pb_subsampled_reproducing}
\min_{f \in \hh, \bB \in \pdm{\RR^{nd}}} J(f) + \Omega(\bB) ~~\st~~ \sum_{p, q=1}^d \dotp{f}{\partial_{pq}k_{x_j}} e_p e_q^T = \Psi_j^T \bB \Psi_j,~~ j \in [n].
\end{equation}
We may then turn \cref{eq:cvx_pb_subsampled_reproducing} into a finite-dimensional optimization problem by considering its dual formulation. This is illustrated on a convex regression problem in \Cref{sec:cvx_reg_exact}.\\

\textit{(ii) Approximate representation.}
While the representation in \cref{eq:exact_rep} allows us to model the Hessian exactly, it may be cumbersome to use in practice. As an example, if the objective functional $J$ contains quadratic terms, solving \cref{eq:cvx_pb_subsampled_reproducing} requires evaluating $\dotp{\partial_{pq}k_{x_i}}{{\partial_{rs}k_{x_j}}}$ for  $p, q, r, s \in [\dimtwo], i, j \in [n]$, which quickly becomes prohibitive as the dimension $\dimtwo$ increases. As an alternative, one may expand $f$ along some features $\{k_{z_i}, i \in [\ell]\}$ (which may be given by the problem, e.g.\ in regression problems), and optimize the weights $\alpha \in \RR^\ell$ such that $f(x) = \sum_{i = 1}^\ell \alpha_i k(x, z_i)$. The constraints are then applied on $\alpha$ by deriving the Hessian of $f$:
\begin{equation}\label{eq:approx_rep}
    \sum_{p, q=1}^d \sum_{i=1}^\ell \alpha_i \frac{\partial^2 k(x_j, z_i)}{\partial x_j^p x_j^q}e_pe_q^T  = \Psi_j^T \bB \Psi_j, ~~j \in [n].
\end{equation}
Compared to the exact representation \cref{eq:exact_rep}, this representation only involves second order derivatives of the kernel, compared to fourth order. Further, it allows controlling the size of the expansion more easily: compared to (ii), using (i) implies using an additional coefficient for each $\partial_{p q} k_{x_j}$, of which there are $n d(d + 1) /2$ (taking symmetry into account).
However, while the pair $(0_\ell, 0_{\ell d \times \ell d})$ always satisfies \cref{eq:approx_rep}, compared to \cref{eq:exact_rep}, some hypothesis is required to ensure that set of pairs $(\alpha, \bB) \in \RR^\ell \times \pdm{\RR^{\ell d}}$ satisfying \cref{eq:approx_rep} has non-empty interior, so that it can be used in practice. Hence, we require the following hypothesis: 
\begin{itemize}
\item[$(H_1)$]\label{hyp:local_strict_cvx} There exists $\alpha \in \RR^\ell$ such that $f(\cdot) \defeq \sum_{i=1}^\ell \alpha_i k(\cdot, z_i)$ is strictly convex in $x_j, j \in [n]$.
\end{itemize}
As examples, for fixed points $(x_1, ..., x_n) \in \RR^d$ and with the Gaussian or the exponential kernel, 
there exists an $\alpha$ satisfying $(H_1)$ if $z_i = x_i, i \in [n]$ and the bandwidth is small enough, or if the kernel is universal and the number of distinct sample points $\ell$ is larger than $n d^2$. 
We leave to finding finer conditions satisfying $(H_1)$ for future work.

\subsubsection{Choosing a representation for smooth convex functions.}
We end this section with a discussion on the choice of PSD SoS to enforce convexity for smooth functions. Compared to \cref{eq:gen_sos_cvx_pb}, several alternative approaches based on scalar kernel SoS that enforce convexity may be considered, which we now briefly examine. Let us motivate our choice based on a few criteria that should be taken into account:
\begin{enumerate}
    \item {\em Representation power}: the class of functions that may be encoded;
    \item {\em Sampling}: the amount of data that must be sampled to learn such a representation;
    \item {\em Scattered data inequalities}: the convexity guarantees one can obtain from subsampled constraints;
    \item {\em Computational cost}: the cost of enforcing those constraints in convex variational problems.
\end{enumerate}
In light of those criteria, let us consider two alternative approaches. The first consists in enforcing that the function $f$ lies above its tangents:
\begin{equation}\label{eq:tangent_constraint}
    \forall x, y \in \Xcal, f(x) - f(y) - \dotp{\nabla f(y)}{x-y} = \dotp{\phi(x, u)}{A\phi(x, u)}_\hh, \quad A \in \pdm{\hh(\Xcal \times\Xcal)},
\end{equation}
where $\hh(\Xcal \times \Xcal)$ denotes a RKHS of functions over $\Xcal \times \Xcal$. The second consists in enforcing that the Hessian of $f$ should have non-negative eigenvalues by representing bilinear products as kernel SoS:
\begin{equation}\label{eq:bilinear_constraint}
    \forall x \in \Xcal, \forall u \in S^{d-1}, u^T H[f](x) u = \dotp{\phi(x, u)}{A\phi(x, u)}_\hh,\quad  A \in \pdm{\hh(\Xcal \times S^{d-1})},
\end{equation}
where $S^{d-1} \subset \RR^d$ is the unit sphere. Following the same proof techniques as in \Cref{thm:cvx_hess_function} and Theorem 5 of \cite{vacher2021dimension}, one may show that \eqref{eq:bilinear_constraint} and \cref{eq:gen_sos_cvx_pb} have the same representation power, whereas \eqref{eq:tangent_constraint} may represent function that are $1$ order less regular (i.e.\ $m > 1 + d/2$ instead of $m > 2 + d/2$). In terms of sampling, \eqref{eq:tangent_constraint} requires covering $\Xcal^2$, \eqref{eq:bilinear_constraint} covering $\Xcal \times S^{d-1}$, while \cref{eq:gen_sos_cvx_pb} only $\Xcal$, which is more efficient and in particular allows obtaining smaller discrete problems. Next, extending the scattered data inequalities from \cite{rudi2020global} (Theorem 13), one may show that \eqref{eq:tangent_constraint} yields $ \forall x, y \in \Xcal,f(x) - f(y) - \dotp{\nabla f(y)}{x-y}\geq - \eta$  with $\eta > 0$, which does not translate to (strong) convexity guarantees, while \eqref{eq:bilinear_constraint} and \cref{eq:gen_sos_cvx_pb} (from Corollary \ref{cor:scattered_convexity}) imply that $f$ is $-\eta$-strongly convex for some $\eta > 0$. Finally, while \eqref{eq:tangent_constraint} and \eqref{eq:bilinear_constraint} rely on scalar-valued kernels compared to matrix-valued kernels, the fact that they require sampling two variables leads to a number of dual variables of the order $O(n^2)$ (if we sample $n$ variables from each domain), compared to $O(n d^2)$ for matrix models. Since in most applications, $d$ is negligible compared to $n$, the PSD matrix model is also more interesting computationally.

%% file: sections/applications.tex
\section{Application to Convex Regression}\label{sec:applications}

In this section, we illustrate the model introduced in \Cref{sec:convex_functions} on a convex regression task. In convex regression, we are given a training set $(x_1, y_1), ..., (x_n, y_n) \in \RR^d \times \RR$, over which we aim to fit a function $f: \RR^d \rightarrow \RR$ according to a least squares loss, under the constraint that $f$ is convex:
\begin{align}\label{eq:cvx_reg}
\begin{split}
&\min_{f : \RR^d \rightarrow \RR}~ \frac{1}{n} \sum_{i=1}^n(y_i - f(x_i))^2 ~~\st ~~f ~~\text{convex}.
\end{split}
\end{align}
This problem was first considered in the 1950's \citep{hildreth1954point}, motivated by applications to economics. The prevalent approach is to solve \eqref{eq:cvx_reg} on the set of piecewise-affine functions~\citep{seijo2011nonparametric}, which amounts to solving the following linearly constrained quadratic program:
\begin{align}\label{eq:cvx_reg_lcqp}
\begin{split}
&\min_{\substack{\theta \in \RR^n\\ \zeta_1, ... \zeta_n \in \RR^d}}~ \frac{1}{n} \sum_{i=1}^n(y_i - \theta_i)^2 ~~\st~~ \theta_i + \zeta_i^T(x_i - x_j) \leq \theta_j,~~i, j \in [n].
\end{split}
\end{align}
In \cref{eq:cvx_reg_lcqp}, $\theta_i$ represents the value of $f$ at $x_i$, and $\zeta_i$ a subgradient of $f$ at $x_i$. Hence, the constraints in \cref{eq:cvx_reg_lcqp} correspond to \cref{eq:tangent_constraint} for piecewise-linear functions. From there, the estimator may be computed at a new point $x$ as $\hat{f}(x) = \max_{i=1, .., n}\{\hat{\theta}_i + \hat{\zeta_i}^T(x - x_i) \}$.
While this method yields a convex function, it may not leverage the smoothness of the data, e.g.\ when the samples are distributed as $Y = f_\star(X) + \varepsilon$ for some smooth function $f_\star$ and noise $\varepsilon$. 

Using the kernel SoS model introduced in \Cref{sec:convex_functions}, we may perform linear regression with smooth functions, under approximate convexity guarantees given by Corollary \ref{cor:scattered_convexity}. We focus here on the approximate representation (ii) of \Cref{sec:convex_functions}, and refer to \Cref{sec:cvx_reg_exact} for the exact representation (i) using reproducing properties of derivatives.
%
That is, the sum-of-square Hessian constraints are applied at given points of $\{v_1, ..., v_\ell\}$, which may e.g.\ be distributed along a grid, or sampled uniformly at random. 
In the experiments in \Cref{subsec:cvx_reg_exp}, we consider the particular case where $\{v_1, ..., v_\ell\} = \{x_1, ..., x_n\}$.
Plugging $f(x) = \sum_{i=1}^n \alpha_i k(x, x_i)$ in \cref{eq:cvx_reg} with the constraints \cref{eq:approx_rep} and with an additional ridge regularization term $\|f\|_\hh^2 = \alpha^T \bK \alpha$, we obtain the following problem.
\begin{align}\label{eq:approx_sos_cvx_reg}
\begin{split}
&\min_{\substack{\alpha \in \RR^n \\ \bB \in \pdm{\RR^{\ell d}}}}~ \frac{1}{n} \sum_{i=1}^n(y_i - \sum_{j=1}^n \alpha_j k(x_i, x_j))^2  +  \lambda_1 \tr \bB + \frac{\la_2}{2}\|\bB\|_F^2 +  \rho \alpha^T \bK \alpha
\\&~~\st~~\forall j \in [\ell], \forall p, q\in [d],~~  \sum_{i=1}^n \alpha_i \frac{\partial^2 k(x_i, v_j)}{\partial v_j^p\partial v_j^q}  = e_p^T(\Psi_j^T \bB \Psi_j)e_q.
\end{split}
\end{align}

\begin{proposition}\label{prop:cvx_reg_approx_dual}
Assuming $(H_1)$, problem \eqref{eq:approx_sos_cvx_reg} admits the following dual formulation:
\begin{align}\label{eq:cvx_reg_approx_dual}
    \min_{\Gamma \in \sym{\RR^d}^\ell} Z(\Gamma)^T(\frac{1}{n} K^2 + \rho K)\inv Z(\Gamma)  + \frac{1}{2\la_2}\|[\sum_{i=1}^\ell \Psi_i\Gamma^{(i)}\Psi_i^T + \lambda_1\eye_{\ell d}]_-\|_F^2,
\end{align}
with $Z(\Gamma) \defeq \frac{1}{2}\sum_{j=1}^\ell\sum_{p, q=1}^d\Gamma^{(j)}_{p, q} \frac{\partial^2K(X, v_j)}{\partial v_j^p\partial v_j^q}  + \frac{1}{n}\bK y \in \RR^n$, $\bK \defeq [k(x_i, x_j)]_{i,j \in [n]} \in \RR^{n\times n}$  and  $$\forall p, q \in [d], j\in [\ell], \frac{\partial^2K(X, v_j)}{\partial v_j^p\partial x_v^q} \defeq \left[\frac{\partial^2K(x_i, v_j)}{\partial v_j^p\partial v_j^q}\right]_{i \in [n]} \in \RR^{n}.$$
\end{proposition}
Proof in \Cref{proof:cvx_reg_approx_dual}, \cpageref{proof:cvx_reg_approx_dual}. 
Problem \eqref{eq:cvx_reg_approx_dual} is smooth and convex, hence it is amenable to accelerated gradient descent. However, its Hessian may have arbitrarily small eigenvalues, therefore it is generally not strongly convex. The computational bottleneck in solving \cref{eq:cvx_reg_approx_dual} is the computation of the negative part, which takes $O(\ell^2 d^3)$ time to form the matrix $\sum_{i=1}^\ell \Psi_i\Gamma^{(i)}\Psi_i^T + \lambda_1\eye_{\ell d}$ and $O(\ell^3 d^3)$ time to compute its SVD. Following~\cite{rudi2021psd} and \cite{muzellec2021note}, we may reduce the computational load by using a Nystrom approximation of $\bK$~\citep[see e.g.][]{williams2001using} to lower the size of the features $\Psi_i, i \in [\ell]$ from  $ \ell d \times d$ to $rd\times d$ with $r<\ell $. This brings down the cost of forming the matrix to $O(\ell rd^3)$ and the cost of the SVD to $O(r^3d^3)$, hence to a total cost of $O(\ell rd^3 + r^3d^3)$ operations per iteration.

\subsection{Numerical Experiments}\label{subsec:cvx_reg_exp}

We illustrate our method on a convex regression task where the data is sampled from a convex function, whose Hessian is PSD but not positive definite, with the addition of noise. More precisely, let $a> 0$ and define $f_a(x) = (\cos(a x) - 1) / a^2 + x^2 / 2$. We have $f_a''(x) = 1 - \cos(ax) \geq 0$. Hence, $f$ is convex, but not strictly so: its second derivative has countably many zeroes. In our experiments, we sample features $X_1, ..., X_n$ uniformly in a hyper cube $[-b, b]^d$, and generate outputs 
\begin{figure}
\centering
\vskip-.5cm
\includegraphics[width=0.39\textwidth]{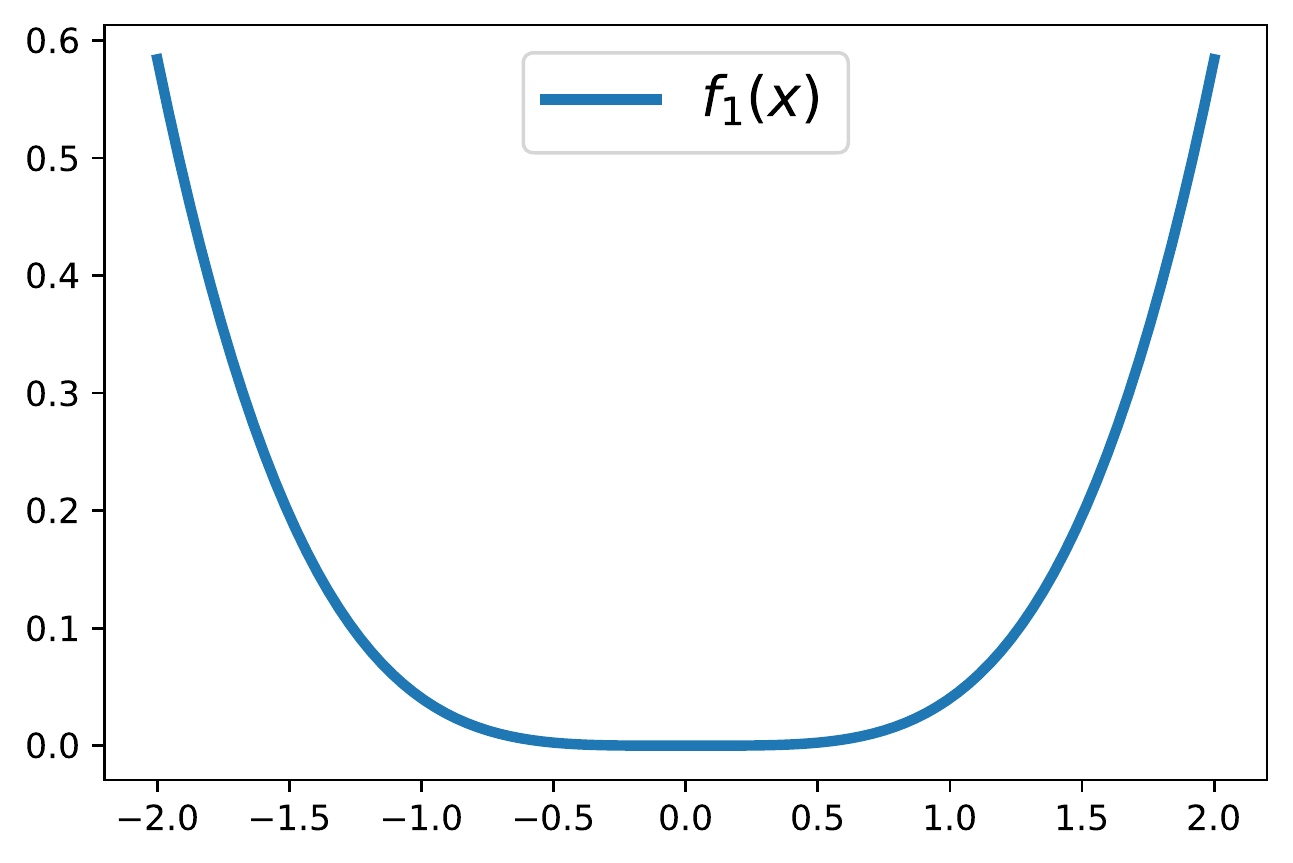}
\caption{Ground truth (1D)}
\label{fig:flat_function}
\end{figure}
\begin{equation}\label{eq:regression_samples}
    Y_i = f_a(\|X_i\|) + \eta\cdot\varepsilon_i,~ i \in [n], ~\text{with}~\eta > 0 ~\text{and}~ \varepsilon_i \underset{iid}{\sim} \Ncal(0, 1),~ i \in [n].
\end{equation}
\begin{figure}
    \centering
    \begin{subfigure}[b]{0.45\textwidth}
         \centering
         \includegraphics[width=\textwidth]{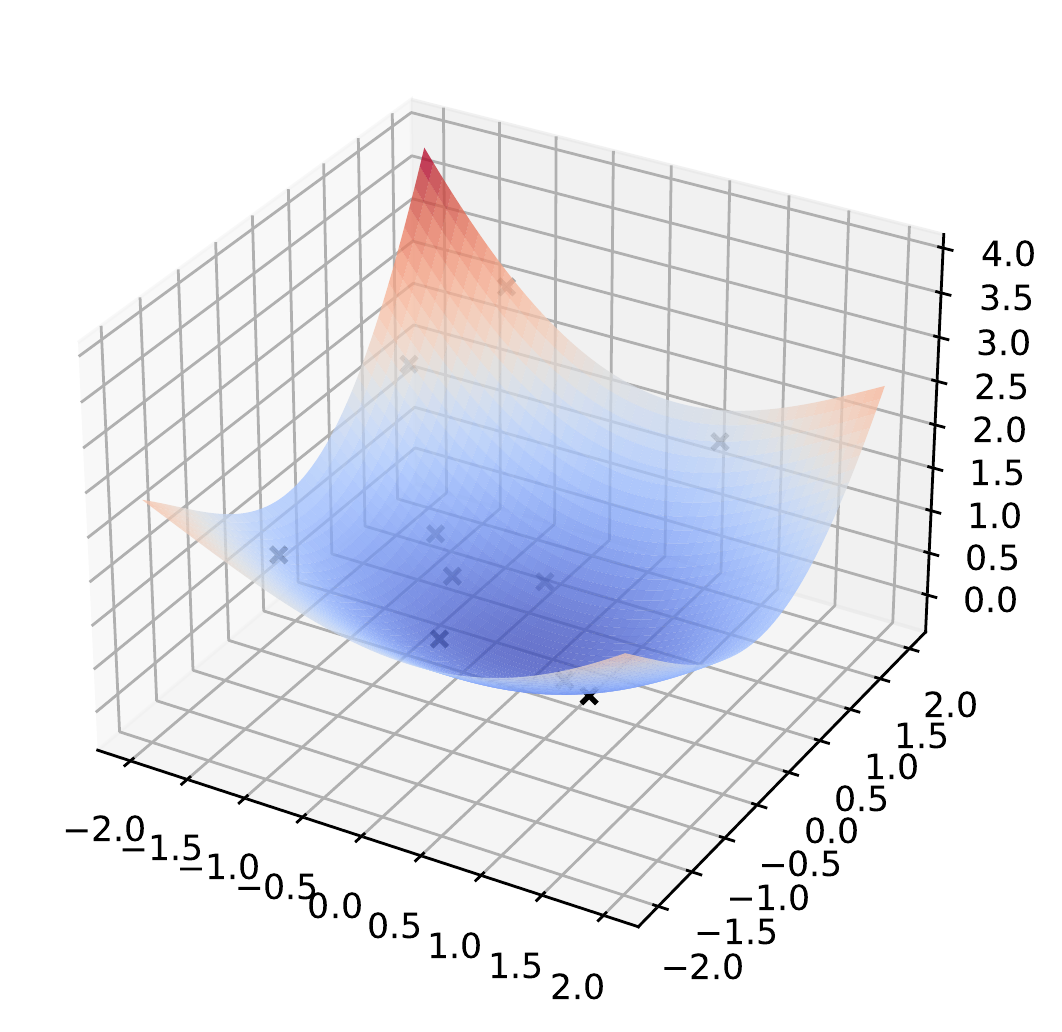}
         \subcaption{Kernel SoS}
     \end{subfigure}
     \hfill
         \begin{subfigure}[b]{0.45\textwidth}
         \centering
         \includegraphics[width=\textwidth]{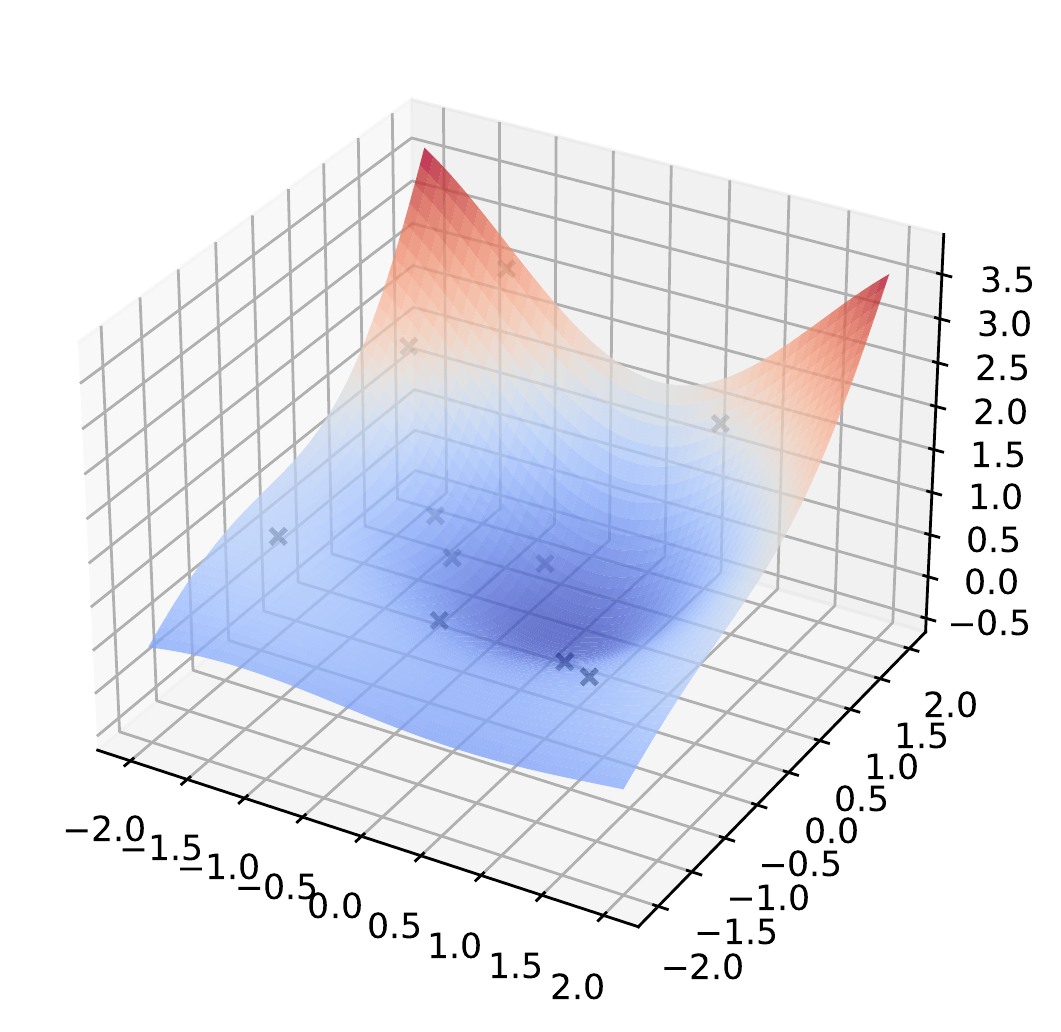}
         \subcaption{Kernel ridge regression}
     \end{subfigure}
     \begin{subfigure}[b]{0.45\textwidth}
         \centering
         \includegraphics[width=\textwidth]{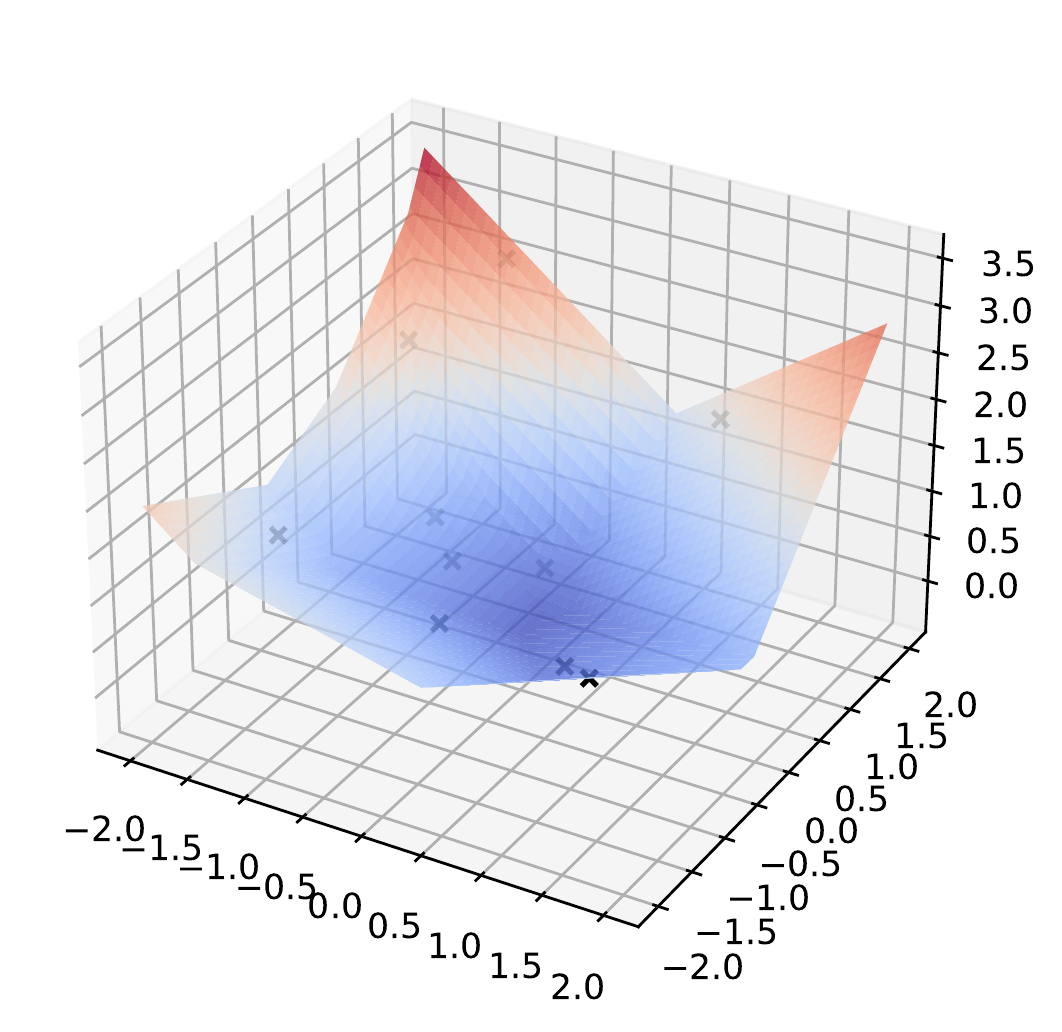}
         \subcaption{Piecewise linear convex regression}
     \end{subfigure}
     \hfill
         \begin{subfigure}[b]{0.45\textwidth}
         \centering
         \includegraphics[width=\textwidth]{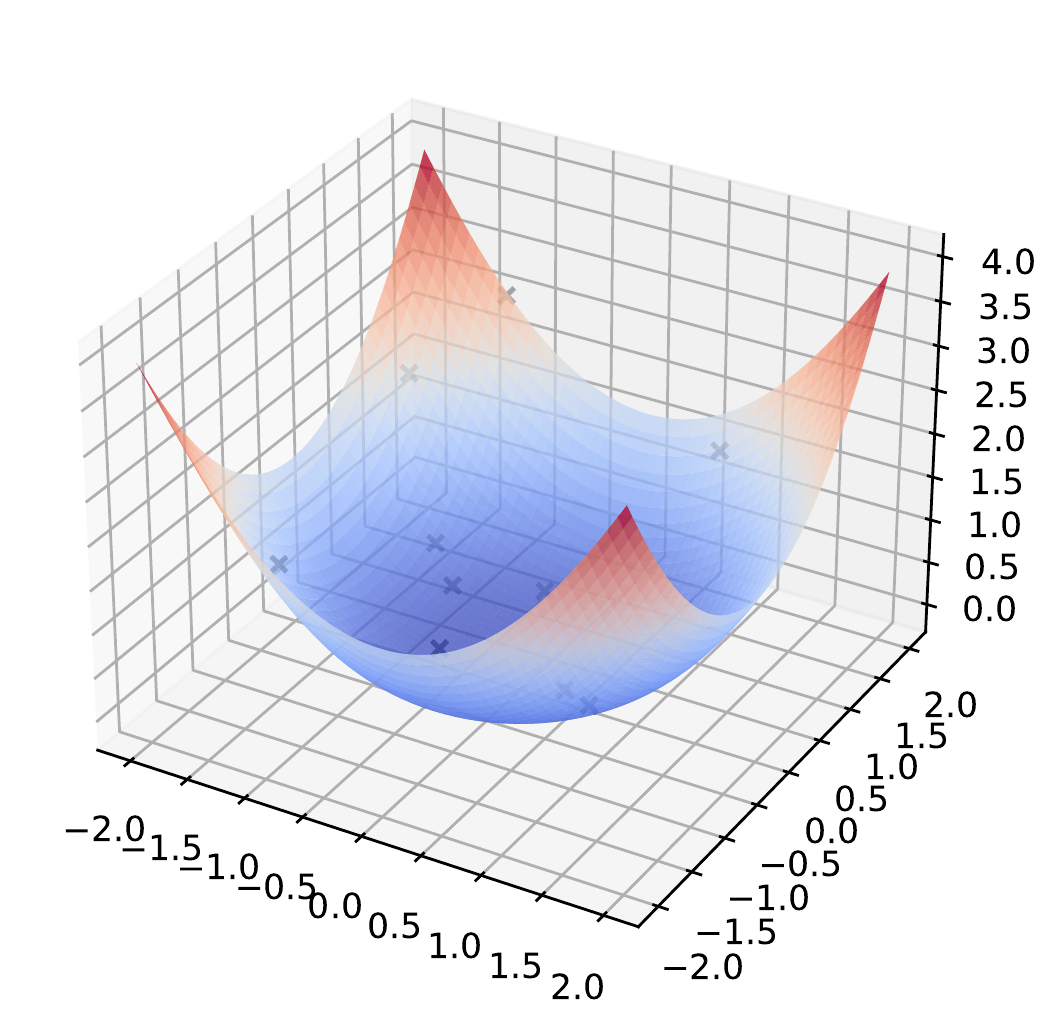}
         \subcaption{Ground truth}
     \end{subfigure}
    \caption{Outcome of kernel convex SoS regression, kernel ridge regression and piecewise linear convex regression fitted on $10$ points compared to ground truth $f_1$. The hyperparameters of kernel SoS and ridge regression are selected using $5$-fold cross-validation.}
    \label{fig:convex_regression}
\end{figure}
 We then fit a regression function by solving \eqref{eq:cvx_reg_approx_dual} with $\ell = n$ and $\{v_1, ..., v_\ell\} = \{x_1, ..., x_n\}$, using the Gaussian kernel and selecting the hyperparameters with $5$-fold cross-validation. We compare kernel SoS models to the piecewise linear convex regression as in \cref{eq:cvx_reg_lcqp}, and to unconstrained kernel ridge regression fitted using cross-validatoin (which may output a non-convex function, as in \Cref{fig:convex_regression}). A sample result is illustrated in \Cref{fig:convex_regression}. In \Cref{fig:MSE_convex_regression}, we report the mean square error of all 3 methods (evaluated by sampling $10000$ additional points according to \cref{eq:regression_samples}), as a function of the number of training samples and of the noise level $\eta$.
 \begin{figure}[ht]
     \centering
    \begin{subfigure}[b]{0.45\textwidth}
         \centering
         \includegraphics[width=\textwidth]{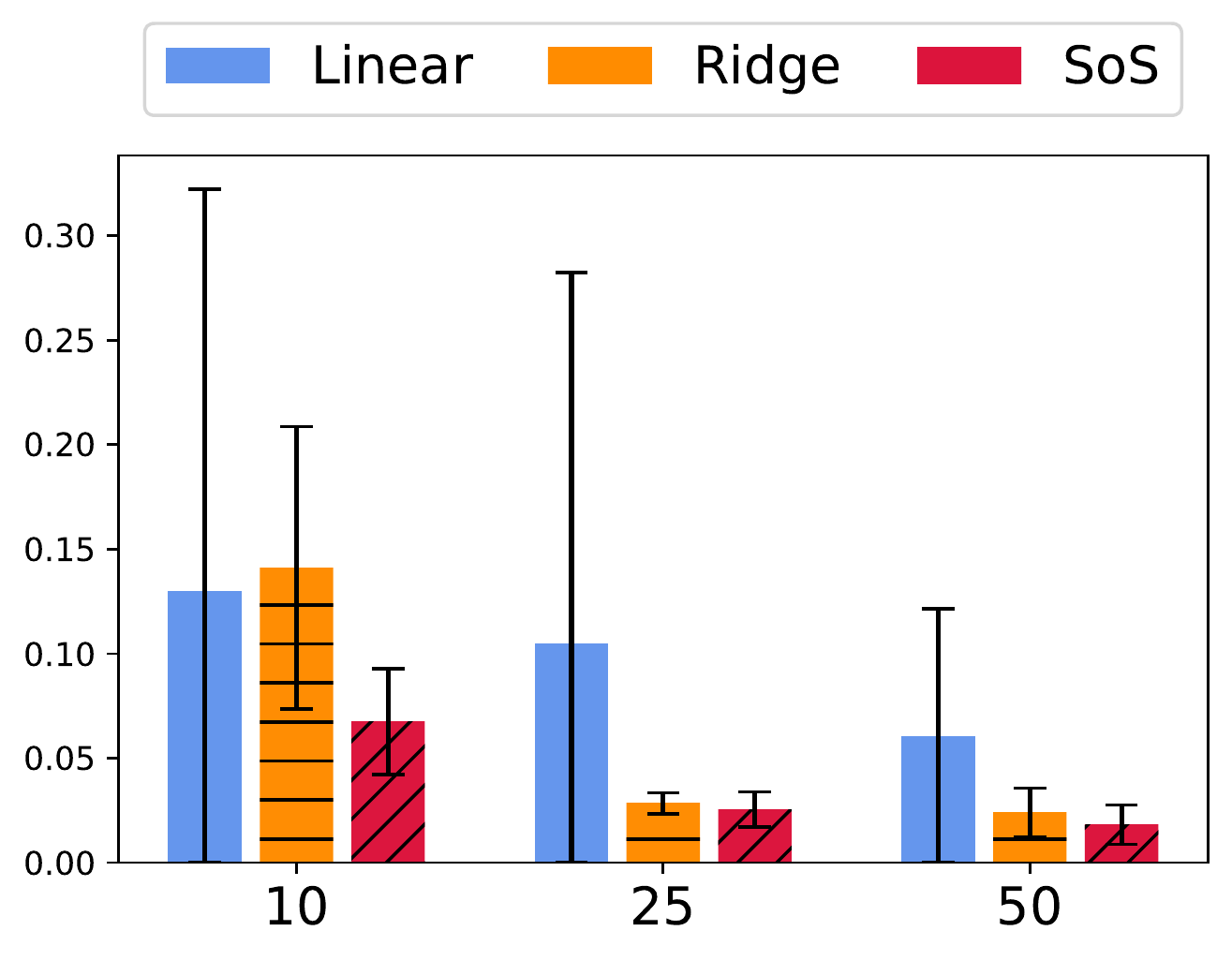}
         \subcaption{$\eta = 0.1$}
     \end{subfigure}
     \hfill
         \begin{subfigure}[b]{0.45\textwidth}
         \centering
         \includegraphics[width=\textwidth]{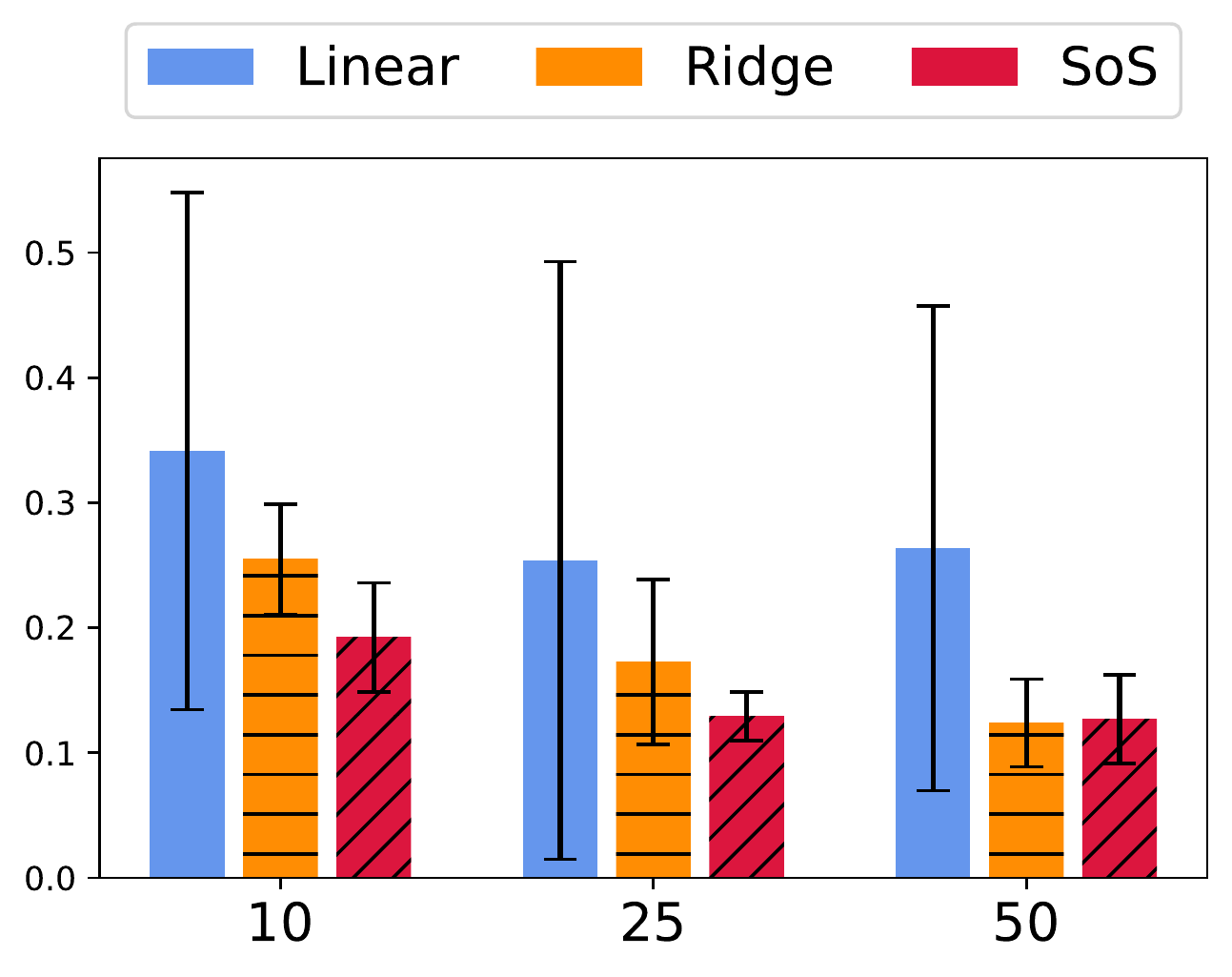}
         \subcaption{$\eta = 0.3$}
     \end{subfigure}
     \caption{Mean square error (MSE) of functions fitted using convex piecewise linear regression, kernel ridge regression and kernel convex sum-of-squares regression on 2D data generated according to \cref{eq:regression_samples} with $a = 1$, as a function of noise ($\eta$) and number of samples (x-axis). Full bars represent averaged scores over 10 runs, error bars correspond to plus/minus standard deviation.}
     \label{fig:MSE_convex_regression}
 \end{figure}
In \Cref{fig:MSE_convex_regression}, the fact that kernel ridge regression performs overall better that piecewise linear convex regression shows that taking smoothness into account allows an improvement of the mean square error. The kernel SoS formulation \eqref{eq:approx_sos_cvx_reg} allows enforcing both smoothness and convexity priors, which yields an additional improvement over kernel ridge regression, as shown in \Cref{fig:MSE_convex_regression}. This is also illustrated in \Cref{fig:convex_regression}: in this example, piecewise linear convex regression yields a convex but non-smooth function, kernel ridge regression yields a smooth but non-convex function, while kernel SoS regression yields a smooth and convex function that generalizes better w.r.t. the ``true'' function.

%% file: sections/conclusion.tex
\section*{Conclusion and Future Work}
In this paper, we have introduced an extension of the kernel SoS models of \cite{marteau2020non} to PSD matrix-valued functions. While we chose to present our model using a particular feature map that relies on a scalar-valued kernel, it straightforwardly extends to more general matrix-valued kernels, as shown in \Cref{sec:vector_rkhs}, which may be of use to leverage the properties of given kernels, such as rotational-free or divergence-free kernels~\citep{macedo2008learning}. We then showed how PSD-valued sums-of-squares could be used in variational problems with convexity constraints, and illustrated our method on a convex regression task. Possible further applications of our method include economics and monopolistic games in particular~\citep{chone2001non,mirebeau2016adaptive}, shape optimization problems such as Newton's least resistance problem~\citep{lachand2005minimizing,merigot2014handling} or optimal transport with a quadratic cost~\citep{brenier1991polar,makkuva2020optimal}.

%% file: sections/appendix.tex
\section{Vector-Valued Functions and Matrix-Valued Kernels}\label{sec:vector_rkhs}

\subsection{Vector-Valued RKHSs}

Let $\Xcal\subset \RR^d$ and $\Ycal$ a Hilbert space. A map $K: \Xcal \times \Xcal \mapsto \Ycal \times \Ycal$ is called a $\Ycal$-reproducing kernel \citep{carmeli2010vector} if $\sum_{i, j=1}^N \dotp{K(x_i, x_j)y_i}{y_j} \geq 0$ for any $x_1, ..., x_N \in \Xcal$ and $y_1, ..., y_N \in \Ycal$. Given a $\Ycal$-reproducing kernel $K$, there is a unique Hilbert space $\hh_K \subset \Ycal^\Xcal$ satisfying 
\begin{enumerate}[(i)]
    \item $\forall x \in \Xcal, K_x \in \Lcal(\Ycal, \hh_K)$ where for $y\in \Ycal$, $K_x y$ is defined by $(K_x y)(t) = K(t, x)y, t \in \Xcal$,
    \item $\forall x \in \Xcal, \forall f \in \hh_K, f(x) = K_x^*f$ where $K_x^* \in \Lcal(\hh_K, \Ycal)$ denotes the adjoint of $K_x$.
\end{enumerate}
In particular, we have $K(x, y) = K_x^*K_y, x, y \in \Xcal$. When $\Ycal \subset \RR^d$, $K(x, y)$ can be identified with a $d\times d$ matrix. Further, under the normality assumption \citep{micchelli2006universal}, $K(x, x)$ is positive definite for all $x \in \Xcal$.
\begin{example}[Matrix-valued kernels]\label{ex:matrix_kernels}
\begin{enumerate}
    \item[]
    \item Separable kernels: $K(x, x') = k(x, x')\bM$ where $\bM \in \pdm{\RR^d}$ encodes the dependence between the outputs. The setting presented in \Cref{sec:pos_op} corresponds to choosing $\bM = \eye_d$ and $k(\cdot, \cdot)$ the kernel associated to the sobolev space $H^s(\Xcal, \RR)$, which already allows encoding functions in $C^m(\Xcal, \RR^d)$ for $s$ large enough.
    \item Non-separable kernels: e.g.\ divergence-free and curl-free kernels~\citep{macedo2008learning}.
\end{enumerate}
\end{example}
Finally, in the proof of \Cref{thm:scattered_inequalities_vector}, which extends \Cref{thm:scattered_inequalities} to more general matrix-valued kernels, we require the following assumption on $\hh_K$, which relates the norms of a vector-valued function $f \in \hh_K $ to its component functions.
\begin{assumption}[Properties of $\hh$ and $\hh_K$]\label{assum:vec_rkhs}
$\hh$ and $\hh_K$ are respectively a scalar-valued and vector-valued RKHS satisfying:
\begin{enumerate}[a)]
    \item For $f = (f^{(1)}, ..., f^{(d)})\in \hh_K$, the $f^{(i)}$'s are in $\hh$;
    \item $\hh$ satisfies \Cref{ass:rkhs} for some $m \in \NN_+$;
    \item\label{item:norm_equiv} There exists $B_K >0$ such that $\sum_{i=1} \|f^{(i)}\|_\hh^2 \leq B_K\|f\|^2_{\hh_K}$.
\end{enumerate}
\end{assumption}
In particular, \Cref{assum:vec_rkhs}\ref{item:norm_equiv} is satisfied when $K(x, x') = k(x, x') \eye_d$, in which case $B_K=1$. More generally, for a separable kernel $K(x, x') = k(x, x') \bM$ (as in \Cref{ex:matrix_kernels}), we may derive bounds from the identity $\|f\|^2_{\hh_K} = \sum_{i,j = 1}^d \bM^{\dagger}_{i,j}\dotp{f^{(i)}}{f^{(j)}}_\hh$, where $\bM^{\dagger}$ is the pseudo-inverse of $\bM$ \citep[see][Section 4.1]{alvarez2012kernels}.

\subsection{Vector-Valued Sums-of-Squares}\label{subsec:vector_sos}

We model positive-operator-valued functions using the following representation:
\begin{equation}\label{eq:gen_sos_model}
F_A(x) = K_x^\star A K_x, ~~\text{with}~~ A \in \pdm{\hh_K}.
\end{equation}
From the definitions above we have that $\forall x, F_A(x) \in \Lcal(\Ycal)$, $F_A$ is self-adjoint and $\forall y \in \Ycal$, $\dotp{y}{F_A(x)y} = \dotp{y}{K_x^\star A K_x y} = \dotp{K_x y}{A K_x y} \geq 0$ by semi positive-definiteness of $A$, hence $\forall x \in \Xcal, F_A(x) \in \pdm{\Ycal}$. \Cref{eq:gen_sos_model} generalizes \cref{eq:model_def_formal}: indeed, the latter can be recovered by taking $\Ycal \subset \RR^d$, $K(x, x') = k(x, x') \eye_d$ and $K_x = \Phi(x)$. 
When $A$ admits a finite-dimensional representation (as obtained from \Cref{thm:rep_thm}) of the form
\begin{equation}
A = \sum_{i,j = 1}^n K_{x_i}\bC_{ij}K_{x_j}^*,
\end{equation}
with $\bC_{ij} \in \RR^{d\times d}, i, j \in [n]$, and $\bC = \begin{pmatrix} \bC_{11} & \dots & \bC_{1n}\\ \vdots & & \vdots \\  \bC_{n 1} & \dots & \bC_{nn} \end{pmatrix} \in \pdm{\RR^{nd}}$,
then, $A$ corresponds to a function of the form
\begin{equation}\label{eq:discrete_form_matrix}
    F_C(x) = \sum_{i,j = 1}^n K(x_i, x)\bC_{ij}K(x_j, x).
\end{equation} 
Note that compared to \cref{eq:discrete_form}, the kernel terms $K(x_i, x)$ and $K(x_j, x)$ are $d\times d$ matrices that do not necessarily commute with $\bC_{ij}$, but satisfy $K(x_i, x)^T = K(x, x_i)$.

Further, as in \Cref{subsec:rep_results}, if $\bK = \left(\begin{smallmatrix} 
K(x_1, x_1) & \cdots & K(x_1, x_n) \\
\vdots & & \vdots\\
K(x_n, x_1) &  & K(x_n, x_n)
\end{smallmatrix} \right) \in \RR^{nd \times nd}$ and $\bR$ denotes the upper-triangular Cholesky factor of $\bK$ (verifying $\bR^T\bR = \bK$), then $\tilde{\bR} = \bR \otimes \eye_d$ is the upper-triangular factor of $\tilde{\bK}$. Letting $\bB = \tilde{\bR} \bC \tilde{\bR}^T$, we then have
$$\tr A = \tr \bB$$
and
$$F_{A}(x_i) = \Psi_i^T \bB \Psi_i, ~~ i \in [n],$$
where $\Psi_i $ denotes the $i$-th $nd \times d$ block column of $\tilde{\bK}$. Further, the model can be queried at new points using the expression
\begin{equation}
F_\bB(x) = \Psi(x)^T\bB\Psi(x),
\end{equation}
with $\Psi(x) = \tilde{\bR}^{-T} v(x)$, and $v(x) = (K(x, x_i))_{i=1}^n \in \RR^{nd\times d}$.

\section{Proofs}\label{sec:PSD_proofs}

For the sake of generality, we write the proofs of the results introduced in \Cref{sec:pos_op} within the matrix-valued kernel framework introduced in \Cref{sec:vector_rkhs}. Taking $\hh_K = \hh^d$, $K_x = \Phi(x)$ (where $\Phi(x)$ is defined in \cref{eq:def_phi}) and $K(x, x') = k(x, x') \eye_d$, we recover the setting of \Cref{sec:pos_op}.

\subsection{Proof of \Cref{prop:lin_and_pos}}\label{proof:lin_and_pos}

\begin{proof}
Let us prove linearity: let $A, B \in \Lcal(\hh_K)$ and $\alpha, \beta \in \RR$. Since $\Lcal(\hh_K)$ is a vector space, it holds $\alpha A + \beta B \in \Lcal(\hh_K)$. Further, we have
\begin{align}
F_{\alpha A + \beta B}(x) &= K_x^\star (\alpha A + \beta B) K_x\\
&= \alpha K_x^\star A K_x + \beta K_x^\star B K_x\\
&= \alpha F_A(x) + \beta F_B(x),
\end{align}
which shows linearity w.r.t. the parameter $A$. Let us now assume that $A \in \pdm{\hh_K}$, let us show that for all $x \in \Xcal$, $F_A(x) \in \pdm{\RR^d}$.
First, let us observe that since $A \in \pdm{\hh_K}$ and $K_x \in \Lcal(\RR^d, \hh_K)$ (as defined in \Cref{sec:vector_rkhs}), we have $F_A(x) \in \Lcal(\hh_K)$. Further, $F_A(x)$ is self-adjoint: 
$$
F_A(x)^\star = K_x^\star A^\star K_x = K_x^\star A K_x = F_A(x).
$$
Let us finally show that $F_A(x)$ is positive semi-definite: let $v \in \RR^d$, we have 
\begin{align}
\dotp{v}{F_A(x)} = \dotp{K_x v }{A K_x v}_{\hh_K} \geq 0,
\end{align}
by positive semi-definiteness of $A$ on $\hh_K$.
\end{proof}

\subsection{Proof of \Cref{thm:rep_thm}}\label{proof:rep_thm}


Following \cite{marteau2020non}, we divide the proof of \Cref{thm:rep_thm} in two main results: \Cref{prop:finite_solution} and \Cref{lemma:matrix_rep}. Let us start with a few definitions: 
\begin{itemize}
    \item Let $\hh_K^n = \Span\{K_{x_i}\}_{i=1}^n = \{\sum_{i=1}^n K_{x_i}y_i:  y_i \in \RR^d, i \in [n]\}$;
    \item $\Pi_n$ denotes the orthogonal projection on $\hh_K^n$;
    \item Finally, define $S_n(\hh_K)_+ = \{\Pi_n A \Pi_n : A \in \pdm{\hh_k}\} \subset \pdm{\hh_K}.$ 
\end{itemize}
Let $\Scal_n: \hh_K \rightarrow \RR^{n\times d}$ defined as 
\begin{equation}
	\Scal_n(h)  = \left(K_{x_i}^\star h\right)_{1 \leq i \leq n}, \quad h \in \hh_K.
\end{equation}
Its adjoint is 
\begin{equation}
	\Scal_n^*(\alpha)  = \sum_{i=1}^n K_{x_i}\alpha_i, \alpha \in \RR^{n\times d}.
\end{equation}
Further, define $\bK = \left(\begin{smallmatrix} 
K(x_1, x_1) & \cdots & K(x_1, x_n) \\
\vdots & & \vdots\\
K(x_n, x_1) &  & K(x_n, x_n)
\end{smallmatrix} \right) \in \RR^{nd \times nd}$
, with rank $r$ ($r = nd$ in the case of a universal kernel \citep{micchelli2006universal}).
Let $\bV \in \RR^{nd \times nd}$ be a matrix such that $\bV^T\bV = \bK$, and define $O_n: \RR^{n} \rightarrow \hh_K$  as $O_n= S_n^* \bV (\bV\bV^T)\inv.$ 
We start with a technical result that will be useful to prove \Cref{lemma:matrix_rep}.
\begin{lemma}[\citet{marteau2020non}, Lemma 4]\label{lemma:id_proj_decomp}
It holds
(i) $O_nO_n^* = \Pi_n$ and (ii) $O_n^*O_n = \Id_r$.
\end{lemma}

\begin{proof}
The proof is identical to that of \citet{marteau2020non}, Lemma 4. We restate it, for the sake of self-inclusiveness.
Since $\bV^T\bV = \bK = \Scal_n\Scal_n^*$, we have 
$$
O_nO_n^* = (\bV\bV^T)\inv\bV S_nS_n^* \bV^T(\bV\bV^T)\inv = (\bV\bV^T)\inv \bV \bV^T \bV \bV^T (\bV\bV^T)\inv = \eye_r,
$$
which proves (ii). Let us now prove (i). Let $\tilde{\Pi}_n = O_nO_n^*$. $\tilde{\Pi}_n $ is a projection operator, since it is self-adjoint and satisfies $\tilde{\Pi}_n^2 = O_n(O_n^*O_n)O_n^* = O_nO_n^* = \tilde{\Pi}_n$, from (ii). Moreover, it follows from the definition of $O_n$ that the range of $\Pi_n$ is included in the range of $\Scal_n^*$, which is $\hh_K^n$. Finally, we observe that $\Rk  \Scal_n^* = \Rk \Scal_n\Scal_n^* =r$, which implies that the dimension of $\hh_K^n$ is $r$. Since $\Rk O_nO_n^* = \Rk O_n^*O_n = r$ from (ii), this implies in turn that $\tilde{\Pi}_n = \Pi_n$.
\end{proof}

We are now ready to prove the main part of \Cref{thm:rep_thm}.

\begin{proposition}[\cite{marteau2020non}, Proposition 7]\label{prop:finite_solution}
Let L be a l.s.c.\ function which is bounded below, and $\Omega$ as in \cref{eq:reg_term}. Then, \cref{eq:gen_sampled_pb} has a solution $A^\star$ that is in $\Scal_n(\hh_K)_+$.
\end{proposition}

\begin{proof}

\textbf{Step 1.} For $A \in \pdm{\hh_K}$, let us denote $J(A) = L(F_A(x_1), ..., F_A(x_n)) + \Omega(A)$, the objective of \eqref{eq:gen_sampled_pb}. We start by showing that  
\begin{equation}
\underset{A \in \Scal_n(\hh_K)_+}{\inf}J(A) = \underset{A \in \pdm{\hh_K}}{\inf}J(A).
\end{equation}
Fix $A \in \pdm{\hh_K}$. We will show that $J(\Pi_n A \Pi_n) \leq J(A )$. First, observe that since $\Pi_n$ is the orthogonal projection on $\Span\{ K_{x_i}, i \in [n]\}$, it holds $K_{x_i} = \Pi_n K_{x_i}$, and therefore 
$$
F_A(x_i) = K_{x_i}^\star A K_{x_i} = K_{x_i}^\star\Pi_n A K_{x_i} = F_{\Pi_n A \Pi_n}(x_i), \quad i \in [n].
$$
In particular, this implies that $ L(F_A(x_1), ..., F_A(x_n)) =  L(F_{\Pi_n A \Pi_n}(x_1), ..., F_{\Pi_n A \Pi_n}(x_n))$.
Next, let us observe that since $\tr A =  \sum_{i \in \NN} \sigma_i$ and $\| A\|^2_\HS =  \sum_{i \in \NN} \sigma_i^2$ where $\sigma_i, i \in [n]$ are the eigenvalues of $A$, we have $\Omega(\Pi_n A \Pi_n) \leq \Omega(A)$.
Putting everything together, this implies that $\forall A \in \pdm{\hh_K}, J(\Pi_n A \Pi_n) \leq J(A).$ Therefore, we have
$$
\underset{A \in \Scal_n(\hh_K)_+}{\inf}J(A) \leq \underset{A \in \pdm{\hh_K}}{\inf}J(A).
$$
Since $\Scal_n(\hh_K)_+ \subset \pdm{\hh_K}$, this finally implies
\begin{equation}
\underset{A \in \Scal_n(\hh_K)_+}{\inf}J(A) = \underset{A \in \pdm{\hh_K}}{\inf}J(A).
\end{equation}

\textbf{Step 2.} Let us now show that $\underset{A \in \Scal_n(\hh_K)_+}{\inf}J(A)$ has a solution. Again, we follow the steps of \cite{marteau2020non}. Let $V_n$ be the injection $\hh_K^n \xhookrightarrow{} \hh_K$. It holds $V_nV_n^* = \Pi_n$ and $V_n^*V_n = \eye_{\hh_K^n}$. Hence, we have 
$$
\Scal_n(\hh_K)_+ = V_n\pdm{\hh_K^n}V_n^* = \{V_n A V_n^* : A \in \pdm{\hh_K^n} \}.
$$
Let us now show that $\underset{A \in \pdm{\hh_K^n}}{\inf} J(V_nAV_n^*)$ has a solution. Let us first observe that $\forall A \in \pdm{\hh_K^n}$, we have $\Omega(V_n A V_n^*)$ \citep[see][Lemma 2]{marteau2020non}, and thus $J(V_n A V_n^*) = L(F_{V_n A V_n^*}(x_1), ..., F_{V_n A V_n^*}(x_n)) + \Omega(A)$. Let $A_0\in \pdm{\hh_K^n}$ be such that $J_0 := J(V_n A_0 V_n^*) < \infty$, and let $c_0$ be a lower bound of $L$. From the \cref{eq:reg_term}, there exists $R_0 >0 $ such that $\|A\|_{\HS} > 0 \implies \Omega(A) > J_0 - c_0$. This implies
$$
\underset{A \in \pdm{\hh_K^n}}{\inf} J(V_nAV_n^*) = \underset{A \in \pdm{\hh_K^n}}{\inf}  J(V_nAV_n^*) ~~\st~~ \|A\|_\HS \leq R_0.
$$
Since $L$ and $\Omega$ are l.s.c., so is $A \mapsto J(V_nAV_n^*)$. Hence, it reaches its minimum on the compact set $\{A \in \pdm{\hh_K^n} : \|A\|_\HS \leq R_0\}$, which is non-empty as it contains $A_0$. Hence, there exists $\tilde{A}_\star \in \pdm{\hh_K^n}$ such that 
$$J(V_n\tilde{A}_\star) = \underset{A \in \pdm{\hh_K^n}}{\inf}  J(V_nAV_n^*) ~~\st~~ \|A\|_\HS \leq R_0,$$
which shows that $\inf_{A \in \pdm{\hh}} J(A) = J(A_\star)$ with $A_\star = V_n\tilde{A}_\star \in \Scal_n(\hh_J)_+.$ 
\end{proof}

To complete the proof of the Theorem, there remains to show the following lemma.

\begin{lemma}\label{lemma:matrix_rep}
\begin{equation}
    \Scal_n(\hh_K)_+ = \left\{\sum_{i,j = 1}^n K_{x_i}\bB_{ij}K_{x_j}^*: \bB = \left(\begin{smallmatrix} \bB_{11} & \dots & \bB_{1n}\\ \vdots & & \vdots \\  \bB_{n 1} & \dots & \bB_{nn} \end{smallmatrix}\right) \in \pdm{\RR^{nd}}\right\}.
\end{equation}

\end{lemma}

\begin{proof}
We proceed by double inclusion. Recall the definitions of $\Scal_n$ and $\Scal_n^*$:

\begin{equation}
	\Scal_n(h)  = \left(K_{x_i}^\star h\right)_{1 \leq i \leq n}, \quad h \in \hh_K,
\end{equation}
and
\begin{equation}
	\Scal_n^*(\alpha)  = \sum_{i=1}^n K_{x_i}\alpha_i, \alpha \in \RR^{n\times d}.
\end{equation}
In particular, for $\bB =  \left(\begin{smallmatrix} \bB_{11} & \dots & \bB_{1n}\\ \vdots & & \vdots \\  \bB_{n 1} & \dots & \bB_{nn} \end{smallmatrix}\right) \in \RR^{nd \times nd}$, it holds
\begin{equation}\label{eq:psd_finite_decomp}
\Scal_n^\star \bB \Scal_n = \sum_{ij = 1}^n K_{x_i} \bB_{ij} K_{x_j}^\star.
\end{equation}
\begin{enumerate}[a)]
\item $S_n(\hh_K)_+ \subset \left\{\sum_{i,j = 1}^n K_{x_i}\bB_{ij}K_{x_j}^*: \bB = \left(\begin{smallmatrix} \bB_{11} & \dots & \bB_{1n}\\ \vdots & & \vdots \\  \bB_{n 1} & \dots & \bB_{nn} \end{smallmatrix}\right) \in \pdm{\RR^{nd}}\right\}$.\\
Let $\Pi_n A\Pi_n \in \Scal_n(\hh_K)_+$, our goal is to show that there exists $\bB \in \pdm{\RR^d}$ such that $\Pi_n A\Pi_n  = \Scal_n^\star \bB \Scal_n.$ Using \Cref{lemma:id_proj_decomp}, we have that $\Pi_n$ can be written $\Scal_n^\star T_n$ with $T_n \in \Lcal(\hh_K, \RR^{n\times d})$. Hence, defining $\bB = T_n A T_n^\star$, we have $\Scal_n^\star \bB \Scal = \Pi_n A \Pi_n$. Further, $A \succeq 0 \implies \bB \succeq 0$. Given \cref{eq:psd_finite_decomp}, this shows the inclusion.
\item $ \left\{\sum_{i,j = 1}^n K_{x_i}\bB_{ij}K_{x_j}^*: \bB = \left(\begin{smallmatrix} \bB_{11} & \dots & \bB_{1n}\\ \vdots & & \vdots \\  \bB_{n 1} & \dots & \bB_{nn} \end{smallmatrix}\right) \in \pdm{\RR^{nd}}\right\} \subset S_n(\hh_K)_+$.\\
Let $\bB \in \pdm{\RR^{nd}}$. It holds $A := \Scal_n^\star \bB \in \pdm{\hh_K}$. Further, since the range of $\Scal_n^\star$ is included in $\hh_K^n$, it holds $\Pi_n S_n^* = \Pi_n$, and therefore $A \in \Scal_n(\hh_K)_+.$
\end{enumerate}
This concludes the proof.
\end{proof}

\paragraph{Proof of the theorem.} The proof of \Cref{thm:rep_thm} follows by combining the two results above: from \Cref{prop:finite_solution}, we have that \cref{eq:gen_sampled_pb} has a solution in $A^* \in S_n(\hh_K)+$, which is unique if $L$ is convex and $\la_2 > 0$. By \Cref{lemma:matrix_rep}, this solution may be written
\begin{equation}\label{eq:finite_form}
A^* = \sum_{i,j = 1}^n K_{x_i}\bB_{ij}K_{x_j}^*,
\end{equation}
for some $\bB = \left(\begin{smallmatrix} \bB_{11} & \dots & \bB_{1n}\\ \vdots & & \vdots \\  \bB_{n 1} & \dots & \bB_{nn} \end{smallmatrix}\right) \in \pdm{\RR^{nd}}$. Further, if $L$ is convex and $\la_2 > 0$, then the objection function of \cref{eq:gen_sampled_pb} is strongly convex, and hence the minimizer is unique.
Finally, plugging \cref{eq:finite_form} in \cref{eq:gen_sos_model}, for $x \in \Xcal$ we have
\begin{align}
    F_{A^*}(x) &= K_x^* A^* K_x\\
    &= \sum_{i,j = 1}^n K_x^*K_{x_i}\bB_{ij}K_{x_j}^*K_x\\
    &= \sum_{i,j = 1}^n K(x_i, x)\bB_{ij}K(x_j, x) =: F_\bB(x).
\end{align}
\Cref{eq:discrete_form} is then a particularization of the expression above to the case $K(x, x') = k(x, x') \eye_d$, where $k(\cdot, \cdot)$ is a scalar-valued kernel.

\subsection{Proof of \Cref{thm:dual_formulation}}\label{proof:dual_formulation}

We start by restating the following useful lemma.
\begin{lemma}[\citet{marteau2020non}, Lemma 5]\label{lemma:omega_dual}
Let $\lambda_1, \lambda_2 \geq 0$. For $\bB \in \sym{\RR^d}$, define
\begin{align}
    \Omega(\bB) = \begin{cases}
    \lambda_1 \tr \bB + \frac{\lambda_2}{2} \|\bB\|_F^2 ~~&\text{if}~~\bB \succeq 0,\\
    +\infty~~&\text{otherwise}.
    \end{cases}
\end{align}
\begin{enumerate}[(i)]
    \item Assume $\la_2 > 0$. Then $\Omega$ is a closed convex function, whose Fenchel conjugate is given for $\bB \in \sym{\RR^d}$ by
\begin{align}
    \Omega^\star(\bB) = 
    \frac{1}{2\lambda_2} \|[\bB - \lambda_1 \eye_d]_+|_F^2,
\end{align}
where $[\bB]_+$ denotes the positive part of $\bB$. $\Omega^\star$ is differentiable and $\frac{1}{\la_2}$-smooth. Its gradient is given by
\begin{equation}
    \nabla \Omega^\star(\bB) = \frac{1}{\lambda_2} [\bB - \lambda_1 \eye_d]_+.
\end{equation}
    \item Assume $\la_2 = 0$ and $\la_1 > 0$. Then $\Omega$ is a closed convex function, whose Fenchel conjugate is given for $\bB \in \sym{\RR^d}$ by
\begin{align}
    \Omega^\star(\bB) = \begin{cases} 0 ~~&\text{if}~~\la_1 \eye_d \succeq \bB,\\
    +\infty~~&\text{otherwise}.
    \end{cases}
\end{align}
\end{enumerate}

\end{lemma}

\begin{proof}
\textbf{(i):} this is exactly \citet{marteau2020non}'s Lemma 5.\\
\textbf{(ii):} Let us adapt the proof \citet{marteau2020non}'s Lemma 5. We may write 
$$\Omega(\bB) = \iota_{\pdm{\RR^d}} + \lambda_1 \tr \bB,$$
where $\iota_{\pdm{\RR^d}}(\bB) = 0$ if $\bB \in \pdm{\RR^d}$ and $+\infty$ otherwise. Since the trace is linear, it is in particular convex, and continuous. Further, $\iota_{\pdm{\RR^d}}$ is closed since $\pdm{\RR^d}$ is a closed non-empty convex subset of $\sym{\RR^d}$. This shows that $\Omega$ is closed and convex, and linear on its domain $\pdm{\RR^d}$. Fix $\bB \in \sym{\RR^d}$, we have
\begin{align}
    \begin{split}
     \Omega^*(\bB) &\defeq \underset{\bA \in \sym{\RR^d}}{\sup} \tr(\bA\bB) - \Omega(\bA) \\
     &= \underset{\bA \in \pdm{\RR^d}}{\sup}\tr(\bA(\bB - \lambda_1 \eye_d))  \\
     &=  \begin{cases} 0 ~~&\text{if}~~\bB - \la_1 \eye_d \preceq 0,\\
    +\infty~~&\text{otherwise}.
    \end{cases}
    \end{split}
\end{align}
\end{proof}

\begin{proof}[\Cref{thm:dual_formulation}]
Like \cite{marteau2020non}, we apply Theorem 3.3.5 of \cite{borwein2006convex} with corresponding arguments represented in \Cref{tab:thm_args}.
\begin{table}[ht]
    \centering
    \begin{tabular}{c|c}
       $\bE$ & $\sym{\RR^{nd}}$  \\
        $\bY$ & $\sym{\RR^d}^n$ \\
        $A: \bE \rightarrow \bY$ & $R : \bB \in \sym{\RR^{nd}} \mapsto (F_\bB(x_1), ..., F_\bB(x_n)) \in \sym{\RR^d}^n$ \\
        $f : \bE \rightarrow (-\infty, +\infty]$ &  $\Omega : \sym{\RR^{nd}} \rightarrow (-\infty, +\infty]$\\
        $g: \bY \rightarrow (-\infty, +\infty]$ & $L: \sym{\RR^d}^n \rightarrow (-\infty, +\infty]$\\
        $p = \inf_{x \in \bE} g(Ax) + f(x)$ & $p = \inf_{\bB \in \sym{\RR^{nd}}} L(F_\bB(x_1), ..., F_\bB(x_n)) + \Omega(\bB)$\\
        $d = \sup_{\phi \in \bY} - g^*(\phi) - f^*(-A^*\phi)$ &  $d = \sup_{\Gamma \in \sym{\RR^d}^{n}} - L^*(\Gamma) - \Omega^*(-R^*(\Gamma))$
    \end{tabular}
    \caption{Corresponding arguments in Theorem 3.3.5 of \citet{borwein2006convex} {\em(left)} and \Cref{thm:dual_formulation} {\em(right)}}
    \label{tab:thm_args}
\end{table}

Indeed, the following properties are satisfied:
\begin{itemize}
    \item L is l.s.c., convex, and bounded below, which implies that it is closed~\citep[see]{borwein2006convex};
    \item $\Omega$ is a non-negative closed convex function, with dual $\Omega^\star$ given in \Cref{lemma:omega_dual}, which is differentiable and smooth when $\la_2 > 0$ (case (i));
    \item The domain of $\Omega$ is $\pdm{\RR^{nd}}$;
    \item $R$ is linear, and for $\Gamma \in \sym{\RR^d}^n$, it holds $R^*\Gamma = \sum_{i=1}^n\Psi_i\Gamma^{(i)}\Psi_i^T$;
    \item Using the expressions of $\Omega^*$ and $R^*$, we may reformulate the dual $d$ in \Cref{tab:thm_args}.
    \begin{enumerate}[(i)]
        \item When $\la_2 > 0$ and $\la_1 \geq 0$:
            $$
                d= \sup_{\Gamma\in \sym{\RR^d}^n} -L^\star(\Gamma) - \frac{1}{2\la_2} \| [\sum_{i=1}^n \Psi_i \Gamma^{(i)} \Psi_i^T + \lambda_1 \eye_{nd}]_{-} \|_F^2,
             $$
        \item When $\la_2 = 0$ and $\la_1 > 0$:
        $$
        d = \sup_{\Gamma\in \sym{\RR^d}^n} -L^\star(\Gamma) ~~\st~~ \sum_{i=1}^n \Psi_i \Gamma^{(i)} \Psi_i^T + \lambda \eye_{nd} \succeq 0;
        $$
    \item Assume there exists $\bB \in \pdm{\RR^{nd}}$ such that $L$ is continuous in $(F_B(x_i))_{1 \leq i\leq n}$. Then there exists a point of continuity of $g$ which is also in $R (\mathrm{dom}(f))$.
    \end{enumerate}
\end{itemize}

We may thus apply theorem 3.3.5 of \cite{borwein2006convex} to get:
\begin{itemize}
    \item $d = p$;
    \item $d$ is attained for a certain $\Gamma_* \in \sym{\RR^d}^n$. Indeed, there exists $\bB \in \mathrm{dom} (\Omega)$ such that $R(\bB) \in \mathrm{dom}(L)$. Thus, $L(R(\bB)) + \Omega(\bB) < + \infty$ and hence $d < + \infty$. Moreover, since $L$ and $\Omega$ are lower-bounded, this shows that $d$ is lower-bounded and hence $d > -\infty$. Hence $d$ is finite and thus is attained by theorem 3.3.5.
\end{itemize}

Finally, if $\la_2 > 0$ and $\la_1 \geq 0$ (case (i)), then $\Omega^*$ is differentiable and since $L$ and $\Omega$ are closed and convex, we may use Exercise 4.2.17 of \cite{borwein2006convex} to show that the primal solution $\bB_\star$ is given by
\begin{equation}
\bB^\star = \nabla \Omega^*(-R^*\Gamma_*) = \frac{1}{\la_2}\Psi\inv[\sum_{i=1}^n \Psi_i \Gamma^{(i)} \Psi_i^T + \lambda_1 \eye_{nd}]_{-} \Psi^{-T}.
\end{equation}

\end{proof}

\subsection{Proof of \Cref{thm:univ_approx}}\label{proof:univ_approx}

We start by restating \Cref{thm:univ_approx} under a more general form that is compatible with the vector-valued RKHSs introduced in \Cref{sec:vector_rkhs}.
\begin{theorem}[Universality]\label{thm:univ_approx_general}
Let $\hh_K$ be a separable Hilbert space of functions taking values in $\RR^d$ and $K: x\in \Xcal \rightarrow K_x \in \hh_K$ a universal map~\citep[see]{micchelli2006universal}. Then 
\begin{equation}
    \{F_A : x \mapsto K_x^\star A K_x  ~~\st~~ A \in \pdm{\hh^d}, \tr A < \infty\}
\end{equation} 
is a universal approximator of $d\times d$ PSD-valued functions over $\Xcal$.
\end{theorem}
\begin{proof}
Let $\Zcal \subset \Xcal$ be a compact set, and $G \in \Ccal(\Zcal, \pdm{\RR^d})$ be a continuous PSD-valued function. 
For $x \in \Zcal$, let $R(x)$ denote the PSD square-root of $G(x)$. Since the matrix square root is continuous on $\pdm{\RR^d}$~\citep{horn2012matrix}, $R \in \Ccal(\Zcal, \pdm{\RR^d})$. For $i, j \in [d]$, define $r_{ij}(x) = e_i^T R(x) e_j$ (where $\{e_i, i \in [d]\}$ is the canonical ONB of $\RR^d$), and $w_i(x) = \sum_{j = 1}^d r_{ij}(x) e_j$. Then, $\forall i \in [d], w_i \in \Ccal(\Zcal, \RR^d)$. Let $\varepsilon >0 $ and $Q = d^2 \sum_{i=1}^d (\varepsilon + 2\|w_i\|_\Zcal)$. By universality of $x \in \Xcal \rightarrow K_x \in \hh_K$, there exist $f_i \in \hh_K, i \in [d]$ such that $\|f_i - g\|_\Zcal \leq \frac{\varepsilon}{Q}, i \in [d]$ (as $f_i(x) = K_x^\star f_i$). Let $A = \sum_{i = 1}^d f_i \otimes f_i \in \pdm{\hh_K}$, and fix $x\in \Zcal$. It holds
\begin{align}
    \begin{split}
    \|F_A(x) - G(x)\|_\infty &\leq d\|F_A(x) - G(x)\|_F\\
     &= d\|\sum_{i=1}^d f_i(x)f_i(x)^T - w_i(x)w_i(x)^T \|_F\\
    & \leq d\sum_{i=1}^d\| f_i(x)f_i(x)^T - w_i(x)w_i(x)^T \|_F\\
    & \leq d\sum_{i=1}^d\|f_i(x) + w_i(x)\|_2 \|f_i(x) - w_i(x)\|_2 \\
    & \leq d^2\sum_{i=1}^d\|f_i(x) + w_i(x)\|_\infty \|f_i(x) - w_i(x)\|_\infty \\
    & \leq d^2\sum_{i=1}^d (\varepsilon + 2\|w_i\|_\Zcal) \|f_i - w_i\|_\Zcal \\
    & \leq \frac{\varepsilon}{Q} d^2\sum_{i=1}^d (\varepsilon + 2\|w_i\|_\Zcal) \leq \varepsilon.
    \end{split}
\end{align}
Since this bound is independent from $x \in \Zcal$, this shows universality.
\end{proof}

\subsection{Proof of \Cref{thm:scattered_inequalities}}\label{proof:scattered_inequalities}

We prove a slightly more general version of \Cref{thm:scattered_inequalities} that is adapted to general vector-valued RKHSs. As for the results above, the original statement can be recovered by taking $K(x, x') = k(x, x')\eye_d$, for which \Cref{assum:vec_rkhs} is satisfied with $B_K = 1$.

\begin{theorem}\label{thm:scattered_inequalities_vector}
Let $\Xcal$ satisfy \Cref{ass:geom_domain} for some $r>0$. Let $K$ be a matrix-valued kernel satisfying \Cref{assum:vec_rkhs}. Let $\hat{X} = \{x_1, ..., x_n\} \subset \Xcal$ such that $h_{\hat{X}, \Xcal} \defeq \underset{x \in \Xcal}{\sup}~\underset{i \in [n]}{\min}~\|x-x_i\| \leq r \min\left(1, \tfrac{1}{18(m-1)^2}\right).$
Let $F \in C^m(\Xcal, \sym{\RR^\dimtwo})$ and assume there exists $\bB \in \pdm{\RR^n}$ such that 
$$F(x_i) = \Psi_i^T \bB \Psi_i, ~~i \in [n],$$
where the $\Psi_i$'s are defined in \Cref{sec:pos_op}. Then, it holds
\begin{equation}
    \forall x \in  \Xcal,~~ \la_{\min}(F(x)) \geq - \varepsilon, ~~\text{where}~~ \varepsilon = C h_{\hat{X},  \Xcal}^m,
\end{equation}
and $C = C_0(\la_{\max}(\bN_F)  + M D_m B_K \tr\bB)$ with $C_0 = 3\frac{\dimone^m}{m!}\max(1, 18(m-1)^2)^m$, and where $\bN_F \in \RR^{\dimtwo\times \dimtwo}$ denotes the matrix whose entries are the pseudo norms of the entries of $F$, $[\bN_F]_{ij} = |F_{ij}|_{\Xcal, m}, {1 \leq i, j \leq \dimtwo} $.
\end{theorem}
\begin{proof}

\textbf{Step 1.} Let $u \in \Scal^{\dimtwo-1}$, and let $g_u(x) = \dotp{u}{\left(F(x) - K_x^*AK_x\right)u}$. From the assumptions, $g_u \in C^m(\Xcal)$. Applying Theorem 13 from \cite{rudi2020global}, we thus have
\begin{equation}
    \sup_{x \in \Xcal} |g_u(x)| \leq \varepsilon_u, ~~\varepsilon_u = c R_m(g_u) h^m_{\hat{X}, \Xcal},
\end{equation}
where $c = 3 \max(1, 18(m-1)^2)^m$ and 
$R_m(g_u) = \sum_{|\alpha| = m} \frac{1}{\alpha!}\sup_{x \in \Xcal}|\partial^\alpha g_u(x)| \leq \frac{\dimone^m}{m!} |g_u|_{\Xcal, m},$
with $$|g_u|_{\Xcal, m} \defeq \underset{|\alpha| = m}{\max}~\underset{x \in \Xcal}{\sup}~|\partial^\alpha g_u(x)|.$$

\textbf{Step 2: uniform bound on $|g_u|_{\hat{X}, \Xcal}, u \in \Scal^{\dimtwo-1}$.}

\begin{enumerate}[1)]
    \item  Let $F_u(x) = \dotp{u}{F(x) u}$. We have $F_u \in C^m(\Xcal)$. Since $F_u(x) = \sum_{1 \leq i , j \leq \dimtwo} u_i u_j F_{ij}(x)$, we have 
\begin{align}
    |F_u|_{\Xcal, m} &= \underset{|\alpha| = m}{\max}~\underset{x \in \Xcal}{\sup}~|\sum_{1 \leq i , j \leq \dimtwo} u_i u_j \partial^\alpha F_{ij}(x)|\\
    & \leq \underset{|\alpha| = m}{\max}~\underset{x \in \Xcal}{\sup}~ \sum_{1 \leq i , j \leq \dimtwo} |u_i| |u_j| |\partial^\alpha F_{ij}(x)|\\
    & \leq \sum_{1 \leq i , j \leq \dimtwo} |u_i| |u_j| |F_{ij}|_{\Xcal, m}\\
    & \leq \lambda_{\max}(\bN_F)
\end{align}
where $\bN_F \in \RR^{d\times d}$, $[\bN_F]_{ij} = |F_{ij}|_{\Xcal, m}, {1 \leq i, j \leq d} $ denotes the matrix whose entries are the pseudo norms of the entries of $F$.
Hence $\sup_{u \in S^{\dimtwo-1}}|F_u|_{\Xcal, m} \leq  \la_{\max} (\bN_F)$.
\item Let $r_{A, u}(x) = \dotp{u}{K_x^*AK_x u}$. We have $r _{A, u}\in C^m(\Xcal)$. Hence, from assumptions on $k$, we have $|r_{A, u}|_{\Xcal, m} \leq D_m \|r_{A, u}\|_\hh$  \citep[see][Remark 2]{rudi2020global}. Let us now bound $\|r_{A, u}\|_\hh$.
$A$ admits an eigen-decomposition $A = \sum_{p \in \NN} \sigma_p f_p \otimes f_p$, with $\{f_p, p \in \NN\}$ an orthonormal family of $\hh_K$, and $\sigma_p, p \in \NN$ is a non-increasing sequence of non-negative scalars converging to zero. Hence, we have
$$ r_{A, u}(x) = \sum_{p \in \NN} \sigma_p u^T f_p(x) u^T f_p(x).$$
Hence, by \Cref{assum:vec_rkhs}\ref{item:norm_equiv},
\begin{align}
    \|r_{A, u}\|_\hh &\leq \sum_{p \in \NN} \sigma_p \|(u^Tf_p(\cdot))^2\|_\hh\\
    &\leq M \sum_{p \in \NN} \sigma_p \|u^Tf_p(\cdot)\|_\hh^2 \\
    &\leq M B_K \tr A,
\end{align}
Indeed, for $f(x) = (f^{(1)}(x), ... ,f^{(\dimtwo)}(x))) \in \hh_K$, we have
\begin{align}
    \|u^Tf\|_\hh &\leq \sum_{s = 1}^d \|u_s f^{(s)}\|_\hh\\
    &\leq \sum_{s = 1}^\dimtwo |u_s| \|f^{(s)}\|_\hh.
\end{align}
Since $u\in\Scal^{d-1}$, the last line is maximized when $|u_s| = \frac{\|f^{(s)}\|_\hh}{\left(\sum_{i=1}^d\|f^{(i)}\|^2_\hh\right)\srt}.$
If $f$ satisfies $\|f\|_{\hh_K}^2 = 1$, this yields
\begin{align}
    \|u^Tf\|_\hh &\leq \left(\sum_{i=1}^d\|f^{(i)}\|^2_\hh\right)\srt \leq B_K \|f\|_{\hh_K} = B_K.
\end{align}
Hence we have $\sup_{x \in \Xcal} |g_u(x)| \leq \tau$ for $\tau =  \la_{\max} (\bN_F) + M D_m B_K \tr A $ that does not depend on $u$. 

\end{enumerate}

\paragraph{Step 3: Conclusion.}
Putting everything together, we have that 
$$
\forall x \in \Xcal, \lambda_{\min} F(x) = \underset{u \in \Scal{d-1}}{\min} F_u(x) \geq - C_0(\la_{\max}(\bN_F) + M D_m B_K \tr A) h^m_{\hat{X}, \Xcal},
$$
with $C_0 = 3\frac{\dimone^m}{m!}\max(1, 18(m-1)^2)^m$. In the statement of \Cref{thm:scattered_inequalities_vector}, we have further bounded $\la_{\max}(\bN_F)$ from above by $\tr \bN_F = \sum_{i=1}^\dimtwo |F_{ii}|_{\Xcal, m}$ for simplicity.
\end{proof}

\section{Modeling Convex Functions: Proofs}\label{sec:cvx_proofs}

\subsection{Proof of \Cref{thm:cvx_hess_function}}\label{proof:cvx_hess_function}
We start by restating \Cref{thm:cvx_hess_function} under a slightly more general form, that allows matrix-valued kernels.

\begin{theorem}[PSD sum-of-squares representation for convex functions]\label{thm:cvx_hess_function_matrix}
Let $\Xcal \subset \RR^d$ and $\hh, \hh_K$ be RKHSs of $\RR$-valued and $\RR^d$-valued functions on $\Xcal$ satisfying \Cref{assum:vec_rkhs}, with $\Ccal^s(\Xcal) \subset \hh$ for some $s \in \NN, s > d/2$. Let $f \in \Ccal^{s+2}(\Xcal)$ be strongly convex.  Then, there exists $A \in \pdm{\hh^d}$ such that $H_f(x) = F_A(x), \forall x \in \Xcal$.  
\end{theorem}

\begin{proof}
Since $f$ is assumed strongly convex, there exists $\rho > 0$ such that $\forall x, \bH_f(x) \succeq \frac{\rho}{2} \Id$. In particular, for all $x\in X$, the matrix $\bH_f(x)$ admits a positive square root $\sqrt{\bH_f(x) }$. Further, since $\sqrt{\cdot}$ is $C^\infty$ on the closed set $\{\bA \in \mathbb{S}_+(\Rd) : \bA \succeq \rho \Id\}$ and $\frac{\partial^2 f}{\partial x_i\partial x_j} \in \Ccal^s(\Xcal)$ for all $i, j \in [d]$, the functions $r_{i,j}: x \mapsto e_i^\top\sqrt{\bH_f(x)}e_j$ are in $C^s(\Xcal)$ for all $i, j \in [d]$ \citep[see Proposition 1 and Assumption 2b of ][]{rudi2020global},  where $(e_1, ..., e_d)$ is the canonical ONB of $\RR^d$. 
Define now the functions $$w_i(x) \defeq \sum_{j=1}^d r_{i,j}(x) e_j, \quad  \forall (x,u) \in X \times \Scal^{d-1}, i \in [d].$$ From the above arguments, it holds that $w_i\in \hh, i\in [d]$, and 
$$\bH_f(x)= \sum_{i=1}^d w_i(x)w_i(x)^T, \quad \forall x \in \Xcal.$$
Following \Cref{ass:rkhs}, we have that the $r_{i,j}$'s belong to $\hh$. Hence, by \Cref{assum:vec_rkhs}, it holds $w_i(\cdot) \in \hh_K, i \in [d]$. Finally, defining $A = \sum_{i=1}^d w_i \otimes w_i \in \pdm{\RR^d}$, we have 
\begin{align}
    \begin{split}
        F_A(x) &= K_x^\star A K_x\\
        &= \sum_{i=1}^d K_x^\star (w_i \otimes w_i) K_x\\
        &= \sum_{i=1}^d (K_x^\star w_i) \otimes (K_x^\star w_i) \\
        &= \sum_{i=1}^d w_i(x) \otimes w_i(x) = \bH_f(x),
    \end{split}
\end{align}
which concludes the proof.
\end{proof}

\subsection{Proof of \Cref{cor:scattered_convexity}}\label{proof:scattered_convexity}

\begin{proof}
This is a direct consequence of \Cref{thm:scattered_inequalities}. Indeed, it holds $H_{f}(x_i) = F_{\bB}(x_i), i \in [n]$, and $H_{f}$ satisfies the hypothesis of \Cref{thm:scattered_inequalities}. Hence, it holds
\begin{align}
    \forall x \in \Xcal, \lambda_{\min}(H_{f}(x)) \geq - \eta~\text{with}~\eta = C h_{\hat{X},  \Xcal}^m > 0,
\end{align}
and $C = C_0(\sum_{i=1}^\dimtwo |(H_{f})_{ii}|_{\Xcal, m} + M D_m B_K \tr\bB)$ with $C_0 = 3\frac{\dimtwo^m}{m!}\max(1, 18(m-1)^2)^m$. In particular, this implies that $f$ is $-\eta$ strongly convex.
\end{proof}

\subsection{Proof of \Cref{thm:subsampling_error}}\label{proof:subsampling_error}

\begin{proof}
Let us consider the following problems:
\begin{align}\label{eq:subsampling_error_proof_pb_1}
    \underset{\substack{f \in \hh \\ A \in \pdm{\hh^d}}}{\min} J(f) ~~\st~~~ H_f(x) = F_A(x), ~~x \in \Xcal,
\end{align}
and its regularized and subsampled version
\begin{align}\label{eq:subsampling_error_proof_pb_2}
    \begin{split}
        \underset{\substack{f \in \hh \\ \bB \in \pdm{\RR^{nd}}}}{\min} J(f) + \rho \|f\|^2_\Fcal + \la\tr\bB ~~ \st~~ H_f(x_j) = F_\bB(x_j) ~~j \in [n],
    \end{split}
\end{align}
with respective solutions $(f_\star, A_\star)$ and $(\hat{f}, \hat{\bB})$. 
$(f_\star, A_\star)$ is an admissible point of \eqref{eq:subsampling_error_proof_pb_2}. This implies that
\begin{align}\label{eq:plugin_ineq}
    J(\hat{f}) + \rho \|\hat{f}\|^2_\Fcal + \la \tr \hat{\bB}  \leq J(f_\star) + \rho \|f_\star\|^2_\Fcal + \la \tr A_\star 
\end{align}
and thus
\begin{align}
    J(\hat{f}) - J(f_\star) \leq \la \tr A_\star +  \rho \|f_\star\|^2_\Fcal - \la \tr \hat{\bB} - \rho \|\hat{f_\star}\|^2_\Fcal\leq \la \tr A_\star +  \rho \|f_\star\|^2_\Fcal.
\end{align}
From \Cref{cor:scattered_convexity}, we have that $\tilde{f} \defeq x \mapsto \hat{f}(x) + \tfrac{\eta}{2}\|x\|^2$ is convex. Further, by Lipschitzness of $J$, it holds $|J(\hat{f}) - J(\tilde{f})| \leq L \tfrac{\eta}{2}N(\|\cdot\|^2)$. Plugging this above, we get 
\begin{align}
    J(\tilde{f}) - J(f_\star) &\leq \la \tr A_\star +  \rho \|f_\star\|^2_\Fcal + |J(\hat{f}) - J(\tilde{f})|\\
    &\leq \la \tr A_\star +  \rho \|f_\star\|^2_\Fcal+  L \tfrac{\eta}{2}C_N.
\end{align}
From \Cref{cor:scattered_convexity}, we have $\eta \geq C_0(\sum_{i=1}^d |(H_{\hat{f}})_{ii}|_{\Xcal, m} + MD_m B_K\tr \hat{\bB})h_{\hat{X}, \Xcal}^m.$
From the definition of $|\cdot|_{\Xcal, m}$ in \cref{eq:def_semi_norm} and \Cref{ass:rkhs}\ref{item:kernel_derivatives_bounded}, we have 
\begin{align}
    \sum_{i=1}^d |(H_{\hat{f}})_{ii}|_{\Xcal, m} \leq \sum_{i=1}^d |\hat{f}|_{\Xcal, m+2} \leq d D_{m+2} \|\hat{f}\|_\Fcal.
\end{align}
Further, since $J(f_\star) \leq J(\hat{f})$, we have from \cref{eq:plugin_ineq} that $\rho \|\hat{f}\|^2_\Fcal + \la \tr \hat{\bB}  \leq \rho \|f_\star\|^2_\Fcal + \la \tr A_\star$. From this we get
\begin{align}
    \tr \hat{\bB} &\leq  \tr A_\star + \frac{\rho}{\la} \|f_\star\|^2_\Fcal\\
    \|\hat{f}\|^2_\Fcal  &\leq \sqrt{\frac{\la}{\rho} +  \|f_\star\|^2_\Fcal}.
\end{align}
Let $R^2 = \frac{\la}{\rho} \tr A_\star +  \|f_\star\|^2_\Fcal$ and $C_1 = \tfrac{1}{2} L C_N C_0 \max(D_m, D_{m+2})$. Using the inequalities above, we get
\begin{align}
    J(\tilde{f}) - J(f_\star) &\leq \rho R^2 +  L \tfrac{\eta}{2}C_N \\
    &\leq  \rho R^2 + \frac{C_1 M B_K h_{\hat{X}, \Xcal}^m}{\la}\rho R^2 + C_1 dh_{\hat{X}, \Xcal}^mR,
\end{align}
which yields \cref{eq:subsampling_error}. When $\rho = \la = C_1 M B_K h_{\hat{X}, \Xcal}^m$, we obtain from the above that 
\begin{align}
    J(\tilde{f}) - J(f_\star)  &\leq  2 \rho R^2 + d \frac{\rho}{MB_K} R\\
    &\leq 2 \rho (R + \frac{d}{4MB_K})^2\\
    &\leq 4\rho (R^2 + \frac{d^2}{16 M^2B_K^2})\\
    &\leq 4MC_1 (\tr A_\star +  \|f_\star\|_\Fcal^2 + \frac{d^2}{M^2 B_K^2}),
\end{align}
which completes the proof.
\end{proof}

\subsection{Proof of \Cref{prop:cvx_reg_approx_dual}}\label{proof:cvx_reg_approx_dual}

\begin{proof}
We divide the proof in two parts: we start by showing that strong duality holds, and then derive the dual formulation.
\paragraph{Strong duality.}
Problem \eqref{eq:approx_sos_cvx_reg} is finite-dimensional and convex, hence it suffices to show that a strictly feasible pair $(\alpha, \bB)$ exists, i.e.\ a pair $(\alpha, \bB) \in \RR^n \times \pdm{\RR^{\ell d}}$ with $\bB \succ 0$ satisfying \cref{eq:approx_rep}. Let us show that this is implied by $(H_1)$. Indeed, $(H_1)$ implies that there exists $\alpha \in \RR^n$ such that the Hessian of $f(\cdot) = \sum_{j=1}^n k(\cdot, x_j)$ is positive-definite in $v_i, i \in [\ell]$. For $i \in [\ell]$, let us denote $\bH_i \in \RR^{d \times d}$ the Hessian of $f$ in $x_i$. By hypothesis, it holds $\bH_i \succ 0, i \in [\ell]$. We must show that there exists $\bB \in \pdm{\RR^{\ell d}}, \bB \succ 0$ such that $\Psi_i^T \bB \Psi_i = \bH_i, i \in [\ell]$. For a universal kernel $k$, we have that $\{\Psi_i, i \in [\ell]\}$ is a linearly independent family of $\RR^{\ell d\times d}$.
For $i \in [n]$, let $\bL_i \in \RR^{\ell d \times n}$ be such that $\bL_i^T\bL_i = \bH_i$. Since $\bH_i \succ 0, i \in [\ell]$, the $\bL_i$'s have rank $d$. Let us further choose the $\bL_i$'s such that $\bL \defeq \begin{psmallmatrix} \bL_1 & | & \dots & | & \bL_n \end{psmallmatrix} \in \RR^{\ell d \times \ell d}$ has full rank. Then, let $\bR \in \RR^{\ell d \times \ell d}$ be such that $\bR \Psi_i = \bL_i, i \in [\ell]$. Such a $\bR$ exists since $\begin{psmallmatrix} \Psi_1 & | & \dots & | & \Psi_n \end{psmallmatrix}\in \RR^{\ell d\times \ell d}$ has full rank if $k$ is universal. Let finally $\bB = \bR^T\bR$. It holds $\Psi_i^T\bB\Psi_i = \Psi_i^T\bR^T \bR\Psi_i = \bL_i^T \bL_i, i \in [\ell]$, and $\bB \in \pdm{\RR^{\ell d}}$. Further, since $\bL$ has full rank by assumption, $\bB$ has full rank, which implies $\bB \succ 0$. Therefore, $(\alpha, \bB)$ is a strictly feasible pair, and strong duality holds.
%
\paragraph{Dual formulation.}
Since strong duality holds, we may apply Lagragian duality. The Lagrangian of \cref{eq:approx_sos_cvx_reg} is
\begin{align}\label{eq:lagrangian_approx_reg}
\begin{split}
    \Lcal(\alpha, \bB, \Gamma) &= \frac{1}{n} \sum_{i=1}^n(y_i - \sum_{j=1}^n \alpha_i k(x_i, x_j))^2  + \rho \alpha^T \bK \alpha + \lambda_1 \tr \bB + \frac{\la_2}{2} \|\bB\|_F^2\\
    &-\sum_{j=1}^\ell\sum_{i=1}^n\sum_{p, q=1}^d \Gamma^{(j)}_{pq} \alpha_i \frac{\partial^2 k(x_i, v_j)}{\partial v_j^p\partial v_j^q} + \sum_{j=1}^\ell\dotp{\Gamma^{(j)}}{\Psi_j^T \bB \Psi_j}_F.
\end{split}
\end{align}
\end{proof}
Let us derive the optimality conditions w.r.t.\ $\alpha$. It holds
\begin{equation}
    \nabla_{\alpha}\Lcal(\alpha, \bB, \Gamma) = \frac{2}{n}\bK (\bK \alpha - y ) + 2 \rho \bK \alpha - \sum_{j=1}^\ell\sum_{p, q=1}^d\Gamma^{(j)}_{pq} \frac{\partial^2K(X, v_j)}{\partial v_j^p\partial v_j^q}.
\end{equation}
Setting the gradient to zero, we get 
\begin{equation}
    \alpha = (\frac{1}{n}\bK^2 + \rho \bK)\inv \left(\frac{1}{n}\bK y + \frac{1}{2}\sum_{j=1}^\ell\sum_{p, q=1}^d \Gamma^{(j)}_{pq} \frac{\partial^2 K(X, v_j)}{\partial v_j^p\partial v_j^q}\right)
\end{equation}
Further, from \Cref{lemma:omega_dual} we have that
\begin{equation}
\underset{\bB\in \pdm{\RR^{\ell d }}}{\inf} \la_1 \tr \bB + \frac{\la_2}{2}\|\bB\|_F^2 + \sum_{i=1}^\ell\dotp{\Gamma^{(i)}}{\Psi_j^T \bB \Psi_j}_F  = -\frac{1}{2\la_2}\|[\sum_{i=1}^\ell \Psi_i\Gamma^{(i)}\Psi_i^T + \lambda_1\eye_{\ell d}]_-\|_F^2.
\end{equation}
Plugging everything in \cref{eq:lagrangian_approx_reg}, we get \cref{eq:approx_sos_cvx_reg}.

\subsection{Convex Regression with Reproducing Property of Derivatives}\label{sec:cvx_reg_exact}

In this section, we consider convex regression with the Hessian representation \eqref{eq:exact_rep}. That is, given samples $(x_i, y_i) \in \Xcal \times \RR, i \in [n]$ and grid points $v_j \in \Xcal, j \in [\ell]$ (that may but need not coincide with the $x_i$'s) we solve the following problem 
\begin{align}\label{eq:pos_hessian_cvx_reg}
\begin{split}
&\min_{\substack{f \in \hh \\ \bB \in \pdm{\RR^{n d}}}}~ \frac{1}{n} \sum_{i=1}^n(y_i - f(x_i))^2  +  \lambda_1 \tr \bB + \frac{\lambda_2}{2} \|\bB\|_F^2 + \rho \|f\|_\hh^2
\\&~~\st~~\sum_{p, q=1}^d \dotp{f}{\partial_{pq}k_{v_j}} e_p e_q^T = \Psi_j^T \bB \Psi_j,~~ j \in [\ell].
\end{split}
\end{align}
The following proposition establishes the dual formulation of \eqref{eq:pos_hessian_cvx_reg}, which is more amenable to optimization.
\begin{figure}[ht]
    \centering
    \includegraphics{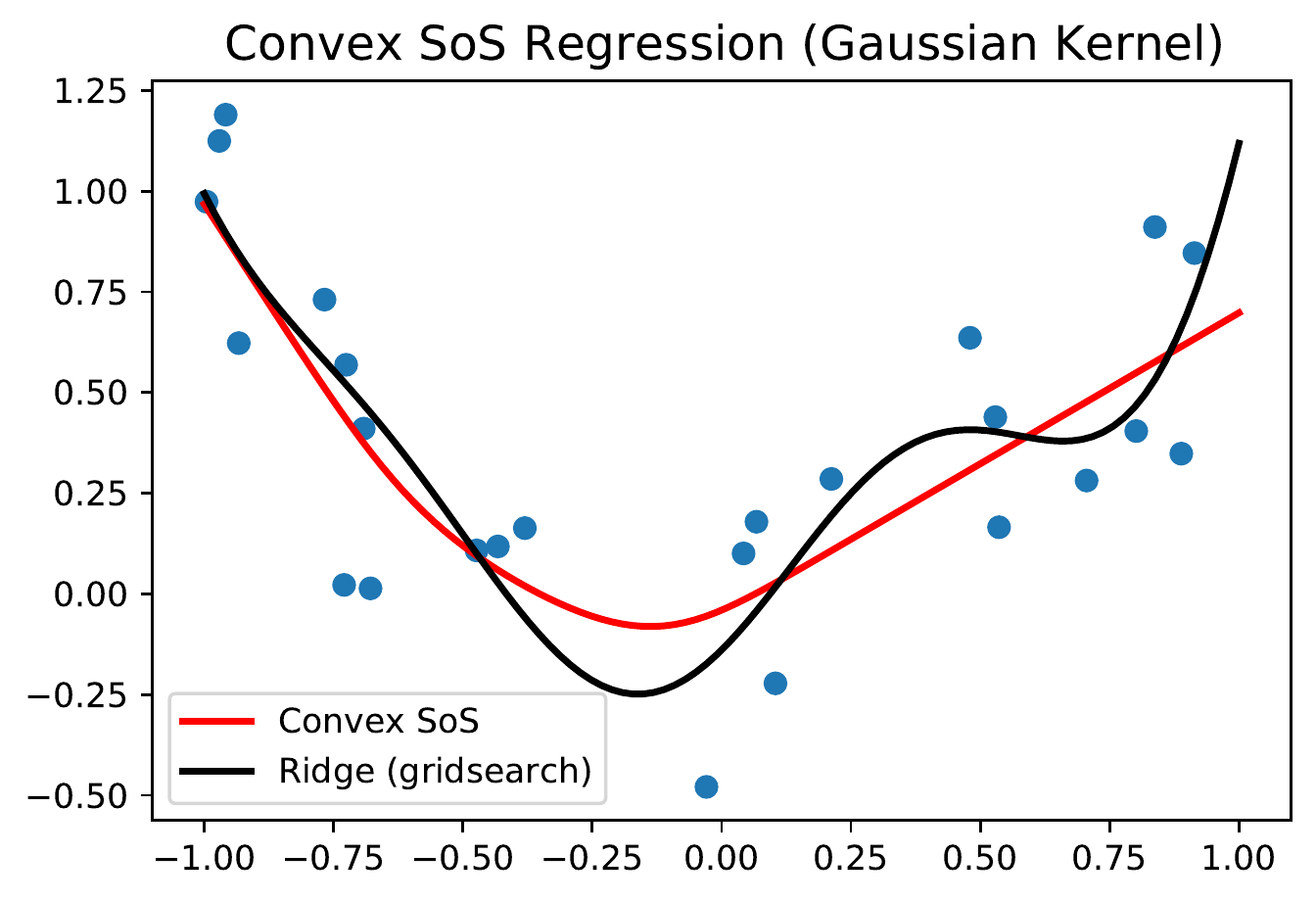}
    \caption{Convex regression with reproducing property: 1D example.}
    \label{fig:cvx_reg}
\end{figure}
\begin{proposition}[Convex regression]\label{prop:cvx_reg}
Let $k$ be a kernel satisfying \Cref{ass:rkhs} and such that $\Ccal^s(\Xcal) \subset \hh$ for some $s\in \NN_+$. Strong duality applies, and \cref{eq:pos_hessian_cvx_reg} admits the following dual formulation:
\begin{align}\label{eq:pos_hess_cvx_reg_dual}
\begin{split}
    &\min_{\Gamma \in \sym{\RR^\dimtwo}^{\ell}}~\Big\{\frac{1}{4\rho}\sum_{p,q,r,s=1}^d\left(\Gamma_{pq}^{(\cdot)}\right)^T\partial_{pqrs}^4\bK(V, V)\Gamma_{rs}^{(\cdot)} 
     \\
     & ~~~~~~~~~~~- \frac{1}{4n\rho}\sum_{p,q,r,s=1}^d\left(\Gamma_{pq}^{(\cdot)}\right)^T\partial_{pq}^2\bK(X, V)^T\bW\partial_{rs}^2\bK(X, V)\Gamma_{rs}^{(\cdot)} \\
    & ~~~~~~~~~~~+ \frac{1}{2n\rho}y^T(\eye_n + \rho \bW - \frac{1}{n}\bK\bW)(\sum_{p,q=1}^d\partial^2_{pq} \bK(X, V)\Gamma_{pq}^{(\cdot)}) \\  & ~~~~~~~~~~~+  \frac{1}{n}y^T (\frac{1}{n}\bK\bW - \eye_n) y + \Omega^*(\sum_{j=1}^\ell \Psi_j \Gamma^{(j)}\Psi_j^T)\Big\},
\end{split}
\end{align}
where for $p, q \in [d],~ \Gamma^{(\cdot)}_{pq} \defeq [\Gamma^{(j)}_{pq}]_{j \in [\ell]}  \in \RR^{\ell},$ and with
\begin{align}
    &\bW \defeq (\frac{1}{n}\bK + \rho\eye_n)\inv, ~~ K_{ij} = k(x_i, x_j),  i, j  \in [n],\\
    &[\partial_{pq}^2\bK(X, V)]_{ij} \defeq \frac{\partial^2 k(x_i, v_j)}{\partial v_j^p \partial v_j^q}, ~~ p, q  \in [d], i \in [n], j \in [\ell],\\
    &[\partial_{pqrs}^4\bK(V, V)]_{ij} \defeq \frac{\partial^4 k(v_i, v_j)}{\partial v_i^p \partial v_i^q\partial v_j^r \partial v_j^s}~~p, q, r, s  \in [d], i \in [n], j \in [\ell].
\end{align}
and 
\begin{align}
\Omega^*(\bB) =
\begin{cases}
     \frac{1}{2\lambda_2} \|[\bB + \lambda_1 \eye_d]_-\|_F^2 & \text{if}~~\la_2 > 0 ~\text{and}~ \la_1 \geq 0,\\ 
    \iota_{\{\bB+ \la_1 \eye_{nd} \succeq 0\}} &\text{if}~~\la_2 = 0 ~\text{and}~ \la_1 > 0.
\end{cases} 
\end{align}
Further, the corresponding primal solution is 
\begin{equation}\label{eq:primal_dual_cvx_reg_exact}
    f^\star = \sum_{i=1}^n(\alpha_i + \frac{1}{2}\sum_{p,q =1}^d \delta_{p, q}^{(i)})k_{x_i} + \frac{1}{2\rho} \sum_{j=1}^\ell \sum_{p,q=1}^d\Gamma^{(j)}_{pq} \partial_{pq} k_{v_j},
\end{equation}
where 
\begin{align}
    & \alpha = \frac{1}{n}\bW y\\
    & \delta_{pq}^{(\cdot)} = -\frac{1}{n\rho} \bW\partial^2_{pq} \bK(X, V) \Gamma_{pq}^{(\cdot)} \in \RR^{n},~~ p,q \in [d].
\end{align}

\end{proposition}

\begin{proof}
We start by showing that strong duality holds, and then derive the dual formulation.\\

\textit{Strong duality.}
Since \eqref{eq:pos_hessian_cvx_reg} is a convex problem, it suffices to show that a strictly feasible point exists.
By assumption, $\hh$ contains $\Ccal^s$ functions. Let $f \in \Ccal^s(\Xcal)$ be such that $f$ is strictly convex in $v_j, j\in [\ell]$ (e.g.\, pick $f$ strongly convex and $\Ccal^s$). Let $\bH_j \defeq \bH_f(v_j) \succ 0, j\in [\ell]$. Following the construction in \Cref{proof:cvx_reg_approx_dual}, there exists $\bB \in \pdm{\RR^{\ell d}}, \bB \succ 0$ such that $\Psi_j\bB\Psi_j = \bH_j, j \in [\ell]$. Hence, $(f, \bB)$ is a strictly feasible point. \\

\textit{Dual formulation.}
Since strong duality holds, we may apply Lagrangian duality. We have
\begin{align}\label{eq:exact_rep_lagrangian}
\begin{split}
    \Lcal(f, \bB, \Gamma) &=  \frac{1}{n} \sum_{i=1}^n(y_i - f(x_i))^2  +  \lambda_1 \tr \bB + \frac{\lambda_2}{2} \|\bB\|_F^2 + \rho \|f\|_\hh^2\\
   & - \dotp{f}{\sum_{j = 1}^\ell\sum_{p, q=1}^d \Gamma^{(j)}_{pq} \partial_{pq}k_{v_j}} + \sum_{j = 1}^\ell \dotp{\Gamma^{(j)}}{\Psi_j^T \bB \Psi_j}_F.
\end{split}
\end{align}
Let us derive the optimality conditions w.r.t. $f$. We have
\begin{align}
    \nabla_f \Lcal(f, \bB, \Gamma) = \frac{2}{n}\sum_{i=1}^n(\dotp{f}{k_{x_i}}_\hh - y_i) k_{x_i} + 2 \rho f - \sum_{j = 1}^\ell\sum_{p, q=1}^d \Gamma^{(j)}_{pq} \partial_{pq}k_{v_j}.
\end{align}
Setting this gradient to zero, we obtain
\begin{align}
   f = \Scal\inv\left(\frac{1}{n}\sum_{i=1}^n y_i k_{x_i} + \frac{1}{2}\sum_{j = 1}^\ell\sum_{p, q=1}^d \Gamma^{(j)}_{pq}, \partial_{pq}k_{v_j})\right)
\end{align}
with $\Scal = \frac{1}{n} \sum_{i=1}^n k_{x_i} \otimes k_{x_i} + \rho \Id$. We can show that
\begin{equation}
    \Scal\inv(\frac{1}{n} \sum_{i=1}^n y_i k_{x_i}) = \sum_{i=1}^n \alpha_i k_{x_i} ~~\text{with}~~ 
\alpha = \frac{1}{n}\bW y,
\end{equation} 
and 
\begin{align}
\Scal\inv(\sum_{j=1}^\ell \sum_{p, q=1}^d \Gamma^{(j)}_{pq} \partial_{pq} k_{v_j}) = \sum_{i=1}^n \delta_{pq}^{(i)} k_{x_i} + \frac{1}{\rho} \sum_{j=i}^\ell \sum_{p, q=1}^d \Gamma^{(j)}_{pq} \partial_{pq} k_{v_i},
\end{align}
with $$\delta_{pq}^{(\cdot)} = -\frac{1}{n\rho} \bW\partial^2_{pq} \bK(X, V) \Gamma_{pq}^{(\cdot)} \in \RR^{n}, p,q \in [d], $$
where $\bW \defeq (\rho \eye_n + \frac{1}{n}\bK)\inv$ and $[\partial_{pq}^2\bK(X, V)]_{ij} \defeq \frac{\partial^2 k(x_i, v_j)}{\partial v_j^p \partial v_j^q}, ~~ i \in [n], j \in [\ell], p, q  \in [d].$
Let us now derive the optimality conditions in $\bB$. We have
\begin{align}
    &\inf_{\bB \in \pdm{\RR^{nd}}}\lambda_1 \tr \bB + \frac{\lambda_2}{2} \|\bB\|_F^2 + \sum_{j = 1}^\ell \dotp{\Gamma^{(j)}}{\Psi_j^T \bB \Psi_j}_F \\
    &= \inf_{\bB \in \pdm{\RR^{nd}}}\lambda_1 \tr \bB + \frac{\lambda_2}{2} \|\bB\|_F^2 + \dotp{\bB}{\sum_{j = 1}^\ell \Psi_j \Gamma^{(j)} \Psi_j^T}_F \\
    &= \Omega^*(\sum_{j=1}^\ell \Psi_j \Gamma^{(j)}\Psi_j^T),
\end{align}
where, from \Cref{lemma:omega_dual}, 
\begin{align}
\Omega^*(\bB) =
\begin{cases}
     \frac{1}{2\lambda_2} \|[\bB + \lambda_1 \eye_d]_-\|_F^2 & \text{if}~~\la_2 > 0 ~\text{and}~ \la_1 \geq 0,\\ 
    \iota_{\{\bB+ \la_1 \eye_{nd} \succeq 0\}} &\text{if}~~\la_2 = 0 ~\text{and}~ \la_1 > 0.
\end{cases} 
\end{align}
  Finally, plugging everything in \cref{eq:exact_rep_lagrangian}, we get the following problem:
\begin{align}
\begin{split}
  \underset{\Gamma \in \sym{\RR^d}^n}{\inf}&\Big\{ \frac{1}{4\rho}\sum_{p,q,r,s=1}^d(\Gamma_{pq}^{(\cdot)})^T\partial_{pqrs}^4\bK(V, V)\Gamma_{rs}^{(\cdot)} \\
    &- \frac{1}{4n\rho}\sum_{p,q,r,s=1}^d\left(\Gamma_{pq}^{(\cdot)}\right)^T\partial_{pq}^2\bK(V, X)\bW\partial_{rs}^2\bK(X, V)\Gamma_{rs}^{(\cdot)} \\
    & + \frac{1}{2n\rho} y^T(\eye_n + \rho\bW - \frac{1}{ n}\bK\bW)(\sum_{p,q=1}^d\partial^2_{pq} \bK(X, V)\Gamma_{pq}^{(\cdot)}) \\
    & + \frac{1}{n}y^T (\frac{1}{n}\bK\bW - \eye_n) y\Big\},  
\end{split}
\end{align}
with $[\partial_{pqrs}^2\bK(V, V)]_{ij} \defeq \frac{\partial^4 k(v_i, v_j)}{\partial v_i^p \partial v_i^q\partial v_j^r \partial v_j^s},~~i, j \in [\ell], p, q, r, s  \in [d]. $
\end{proof}

\section{Additional Experiments and Numerical Details}\label{sec:add_exp}

\begin{figure}[ht]
    \centering
    \includegraphics[width=\textwidth]{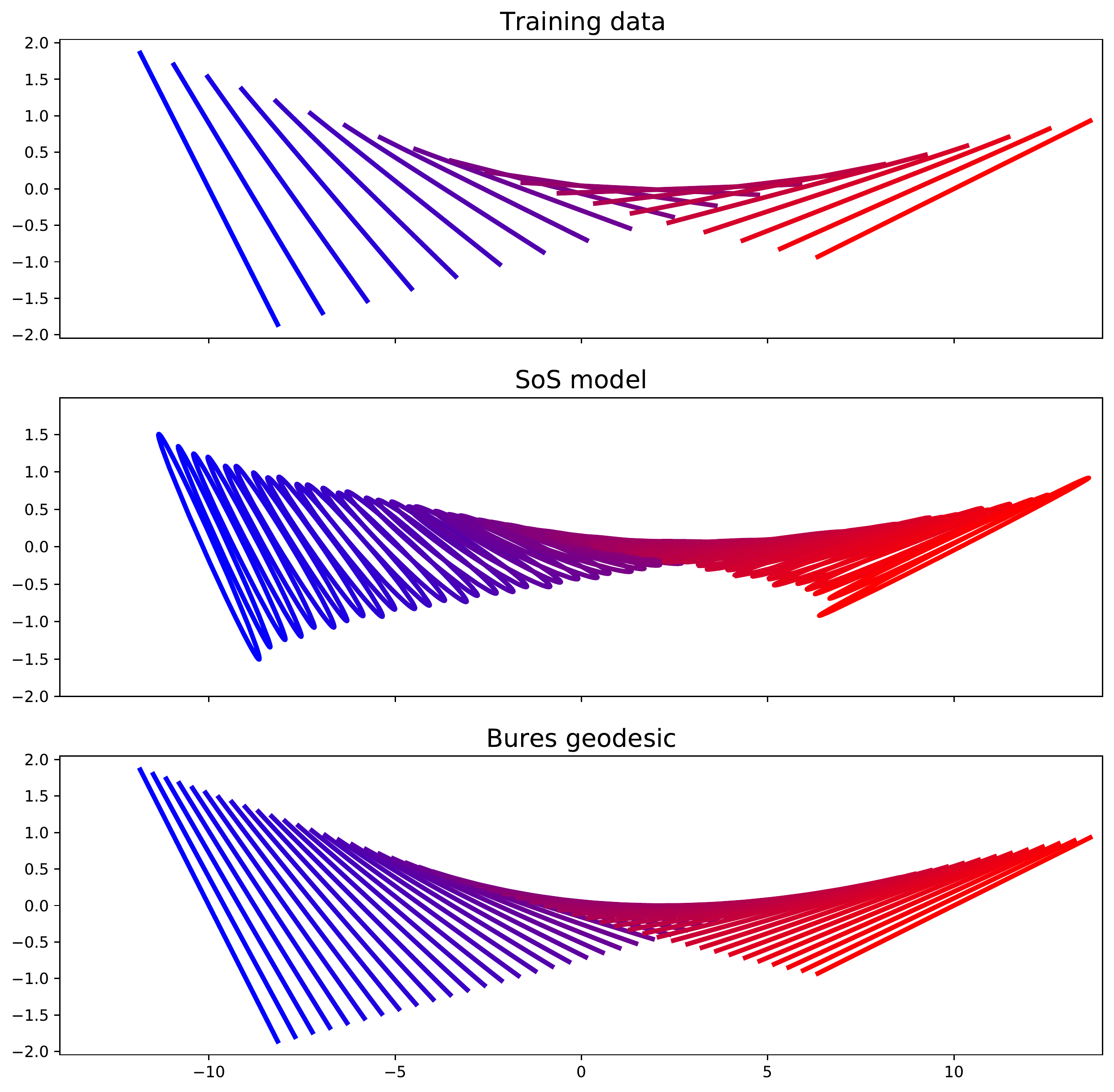}
    \caption{Interpolation of a rank $1$ Bures geodesic from $12$ data points (top). The matrices are represented as the level sets of (potentially degenerate) Gaussian distributions: $\{x : \Ncal(x; 0, \Sigma) \leq r\}$. We use the exponential kernel and select all hyperparameters using cross-validation. The learned model is represented in the middle figure, and the full geodesic is plotted in the bottom figure for comparison.}
    \label{fig:bures_geodesic_rank_1}
\end{figure}

\subsection{PSD-Valued Least Squares Regression}
We select hyperparameters $\lambda_1, \lambda_2$ and exponential kernel bandwidth $\sigma$ using leave-one-out cross-validation, over the following grid: $\la_1 \in \{10^{-n}, n = 0, ..., 8\}\cup \{0\}$, $\la_2 \in \{10^{-n}, n = 0, ..., 8\}$, $\sigma \in \{1, 0.1, 0.01\}$. In \Cref{fig:bures_geodesic}, the parameters selected by CV are $\la_1 = 0, \la_2 = 10^{-5}, \sigma = 0.1$ and in \Cref{fig:bures_geodesic_rank_1}, $\la_1 = 0, \la_2 = 10^{-7}, \sigma = 0.1$.

\subsection{Convex Regression}

We select hyperparameters $\rho, \lambda_2$ and Gaussian kernel bandwidth $\sigma^2$ using $5$-fold cross-validation, over the following grid: $\la_2 \in \{10^{-n}, n = 3, ..., 7\}$, $\sigma^2 \in \{1, 5, 10\}$. $\la_1$ is fixed to $0$. When $n > 25$, we perform Nyström approximation with rank $r = 25$. In \Cref{fig:MSE_convex_regression}, we display the average scores and their standard deviations on 10 independent sets of samples. In \Cref{fig:convex_regression}, the hyperparameters selected by CV are $\la_2 = 10^{-3}, \rho = 10^{-5}$ and $\sigma^2 = 10$.

%% file: main.bbl
\begin{thebibliography}{34}
\providecommand{\natexlab}[1]{#1}
\providecommand{\url}[1]{\texttt{#1}}
\expandafter\ifx\csname urlstyle\endcsname\relax
  \providecommand{\doi}[1]{doi: #1}\else
  \providecommand{\doi}{doi: \begingroup \urlstyle{rm}\Url}\fi

\bibitem[Adams and Fournier(2003)]{adams2003sobolev}
Robert~A. Adams and John J.~F. Fournier.
\newblock \emph{{Sobolev Spaces}}.
\newblock Elsevier, 2003.

\bibitem[{\'A}lvarez et~al.(2012){\'A}lvarez, Rosasco, and
  Lawrence]{alvarez2012kernels}
Mauricio~A {\'A}lvarez, Lorenzo Rosasco, and Neil~D Lawrence.
\newblock Kernels for vector-valued functions: A review.
\newblock \emph{Foundations and Trends{\textregistered} in Machine Learning},
  4\penalty0 (3):\penalty0 195--266, 2012.

\bibitem[Amos et~al.(2017)Amos, Xu, and Kolter]{amos2017input}
Brandon Amos, Lei Xu, and J~Zico Kolter.
\newblock Input convex neural networks.
\newblock In \emph{International Conference on Machine Learning}, pages
  146--155. PMLR, 2017.

\bibitem[Aubin-Frankowski and Szabo(2021)]{aubin2021handling}
Pierre-Cyril Aubin-Frankowski and Zoltan Szabo.
\newblock Handling hard affine {SDP} shape constraints in {RKHS}s.
\newblock \emph{arXiv preprint arXiv:2101.01519}, 2021.

\bibitem[Berthier et~al.(2021)Berthier, Carpentier, Rudi, and
  Bach]{berthier2021infinite}
Elo{\"\i}se Berthier, Justin Carpentier, Alessandro Rudi, and Francis Bach.
\newblock Infinite-dimensional sums-of-squares for optimal control.
\newblock \emph{arXiv preprint arXiv:2110.07396}, 2021.

\bibitem[Bhatia et~al.(2019)Bhatia, Jain, and Lim]{bhatia2019bures}
Rajendra Bhatia, Tanvi Jain, and Yongdo Lim.
\newblock On the {B}ures--{W}asserstein distance between positive definite
  matrices.
\newblock \emph{Expositiones Mathematicae}, 37\penalty0 (2):\penalty0 165--191,
  2019.

\bibitem[Borwein and Lewis(2006)]{borwein2006convex}
Jonathan Borwein and Adrian~S Lewis.
\newblock \emph{{Convex Analysis and Nonlinear Optimization: Theory and
  Examples}}.
\newblock Springer Science \& Business Media, 2006.

\bibitem[Brenier(1991)]{brenier1991polar}
Yann Brenier.
\newblock Polar factorization and monotone rearrangement of vector-valued
  functions.
\newblock \emph{Communications on pure and applied mathematics}, 44\penalty0
  (4):\penalty0 375--417, 1991.

\bibitem[Carmeli et~al.(2010)Carmeli, De~Vito, Toigo, and
  Umanit{\'a}]{carmeli2010vector}
Claudio Carmeli, Ernesto De~Vito, Alessandro Toigo, and Veronica Umanit{\'a}.
\newblock Vector valued reproducing kernel {Hilbert} spaces and universality.
\newblock \emph{Analysis and Applications}, 8\penalty0 (01):\penalty0 19--61,
  2010.

\bibitem[Chen et~al.(2018)Chen, Shi, and Zhang]{chen2018optimal}
Yize Chen, Yuanyuan Shi, and Baosen Zhang.
\newblock Optimal control via neural networks: A convex approach.
\newblock In \emph{International Conference on Learning Representations}, 2018.

\bibitem[Chon{\'e} and Le~Meur(2001)]{chone2001non}
Philippe Chon{\'e} and Herv{\'e}~VJ Le~Meur.
\newblock Non-convergence result for conformal approximation of variational
  problems subject to a convexity constraint.
\newblock 2001.

\bibitem[Curmei and Hall(2020)]{curmei2020shape}
Mihaela Curmei and Georgina Hall.
\newblock Shape-constrained regression using sum of squares polynomials.
\newblock \emph{arXiv preprint arXiv:2004.03853}, 2020.

\bibitem[Hildreth(1954)]{hildreth1954point}
Clifford Hildreth.
\newblock Point estimates of ordinates of concave functions.
\newblock \emph{Journal of the American Statistical Association}, 49\penalty0
  (267):\penalty0 598--619, 1954.

\bibitem[Horn and Johnson(2012)]{horn2012matrix}
Roger~A Horn and Charles~R Johnson.
\newblock \emph{{Matrix Analysis}}.
\newblock Cambridge university press, 2012.

\bibitem[Lachand-Robert and Oudet(2005)]{lachand2005minimizing}
Thomas Lachand-Robert and {\'E}douard Oudet.
\newblock Minimizing within convex bodies using a convex hull method.
\newblock \emph{SIAM Journal on Optimization}, 16\penalty0 (2):\penalty0
  368--379, 2005.

\bibitem[Mac{\^e}do and Castro(2008)]{macedo2008learning}
Ives Mac{\^e}do and Rener Castro.
\newblock Learning divergence-free and curl-free vector fields with
  matrix-valued kernels.
\newblock Technical report, Instituto Nacional de Matematica Pura e Aplicada,
  2008.

\bibitem[Makkuva et~al.(2020)Makkuva, Taghvaei, Oh, and
  Lee]{makkuva2020optimal}
Ashok Makkuva, Amirhossein Taghvaei, Sewoong Oh, and Jason Lee.
\newblock Optimal transport mapping via input convex neural networks.
\newblock In \emph{International Conference on Machine Learning}, pages
  6672--6681. PMLR, 2020.

\bibitem[Marteau-Ferey et~al.(2020)Marteau-Ferey, Bach, and
  Rudi]{marteau2020non}
Ulysse Marteau-Ferey, Francis Bach, and Alessandro Rudi.
\newblock Non-parametric models for non-negative functions.
\newblock In \emph{Advances in Neural Information Processing Systems},
  volume~33, pages 12816--12826. Curran Associates, Inc., 2020.

\bibitem[M{\'e}rigot and Oudet(2014)]{merigot2014handling}
Quentin M{\'e}rigot and Edouard Oudet.
\newblock Handling convexity-like constraints in variational problems.
\newblock \emph{SIAM Journal on Numerical Analysis}, 52\penalty0 (5):\penalty0
  2466--2487, 2014.

\bibitem[Micchelli et~al.(2006)Micchelli, Xu, and
  Zhang]{micchelli2006universal}
Charles~A Micchelli, Yuesheng Xu, and Haizhang Zhang.
\newblock Universal kernels.
\newblock \emph{Journal of Machine Learning Research}, 7\penalty0 (12), 2006.

\bibitem[Mirebeau(2016)]{mirebeau2016adaptive}
Jean-Marie Mirebeau.
\newblock Adaptive, anisotropic and hierarchical cones of discrete convex
  functions.
\newblock \emph{Numerische Mathematik}, 132\penalty0 (4):\penalty0 807--853,
  2016.

\bibitem[Muzellec et~al.(2021)Muzellec, Bach, and Rudi]{muzellec2021note}
Boris Muzellec, Francis Bach, and Alessandro Rudi.
\newblock A note on optimizing distributions using kernel mean embeddings.
\newblock \emph{arXiv preprint arXiv:2106.09994}, 2021.

\bibitem[Nesterov(2003)]{nesterov2003introductory}
Yurii Nesterov.
\newblock \emph{{Introductory Lectures on Convex Optimization: A Basic
  Course}}, volume~87.
\newblock Springer Science \& Business Media, 2003.

\bibitem[Paulsen and Raghupathi(2016)]{paulsen2016introduction}
Vern~I. Paulsen and Mrinal Raghupathi.
\newblock \emph{{An Introduction to the Theory of Reproducing Kernel Hilbert
  Spaces}}, volume 152.
\newblock Cambridge University Press, 2016.

\bibitem[Rochet and Chon{\'e}(1998)]{rochet1998ironing}
Jean-Charles Rochet and Philippe Chon{\'e}.
\newblock Ironing, sweeping, and multidimensional screening.
\newblock \emph{Econometrica}, pages 783--826, 1998.

\bibitem[Rudi and Ciliberto(2021)]{rudi2021psd}
Alessandro Rudi and Carlo Ciliberto.
\newblock {PSD} representations for effective probability models.
\newblock \emph{arXiv preprint arXiv:2106.16116}, 2021.

\bibitem[Rudi et~al.(2020)Rudi, Marteau-Ferey, and Bach]{rudi2020global}
Alessandro Rudi, Ulysse Marteau-Ferey, and Francis Bach.
\newblock Finding global minima via kernel approximations.
\newblock In \emph{Arxiv preprint arXiv:2012.11978}, 2020.

\bibitem[Sch{\"o}lkopf et~al.(2002)Sch{\"o}lkopf, Smola, Bach,
  et~al.]{scholkopf2002learning}
Bernhard Sch{\"o}lkopf, Alexander~J Smola, Francis Bach, et~al.
\newblock \emph{{Learning with Kernels: Support Vector Machines,
  Regularization, Optimization, and Beyond}}.
\newblock MIT press, 2002.

\bibitem[Seijo and Sen(2011)]{seijo2011nonparametric}
Emilio Seijo and Bodhisattva Sen.
\newblock Nonparametric least squares estimation of a multivariate convex
  regression function.
\newblock \emph{The Annals of Statistics}, 39\penalty0 (3):\penalty0
  1633--1657, 2011.

\bibitem[Steinwart and Christmann(2008)]{steinwart2008support}
Ingo Steinwart and Andreas Christmann.
\newblock \emph{{Support Vector Machines}}.
\newblock Springer Science \& Business Media, 2008.

\bibitem[Vacher et~al.(2021)Vacher, Muzellec, Rudi, Bach, and
  Vialard]{vacher2021dimension}
Adrien Vacher, Boris Muzellec, Alessandro Rudi, Francis Bach, and
  Francois-Xavier Vialard.
\newblock A dimension-free computational upper-bound for smooth optimal
  transport estimation.
\newblock \emph{Conference on Learning Theory}, 2021.

\bibitem[Wendland(2004)]{wendland2004scattered}
Holger Wendland.
\newblock \emph{{Scattered Data Approximation}}, volume~17.
\newblock Cambridge university press, 2004.

\bibitem[Williams and Seeger(2001)]{williams2001using}
Christopher Williams and Matthias Seeger.
\newblock Using the {N}ystr{\"o}m method to speed up kernel machines.
\newblock In \emph{Advances in Neural Information Processing Systems 13}.
  Citeseer, 2001.

\bibitem[Zhou(2008)]{zhou2008derivative}
Ding-Xuan Zhou.
\newblock Derivative reproducing properties for kernel methods in learning
  theory.
\newblock \emph{Journal of computational and Applied Mathematics}, 220\penalty0
  (1-2):\penalty0 456--463, 2008.

\end{thebibliography}
